\documentclass[journal]{vgtc}                % final (journal style)

\ifpdf%                                % if we use pdflatex
  \pdfoutput=1\relax                   % create PDFs from pdfLaTeX
  \usepackage{graphicx}                % allow us to embed graphics files
  \DeclareGraphicsExtensions{.pdf,.png,.jpg,.jpeg} % for pdflatex we expect .pdf, .png, or .jpg files
\else%                                 % else we use pure latex
  \usepackage{graphicx}                % allow us to embed graphics files
  \DeclareGraphicsExtensions{.pdf}     % for pure latex we expect eps files
\fi%

\graphicspath{{figures/}{pictures/}{images/}{./}} % where to search for the images

\usepackage{microtype}                 % use micro-typography (slightly more compact, better to read)
\PassOptionsToPackage{warn}{textcomp}  % to address font issues with \textrightarrow
\usepackage{textcomp}                  % use better special symbols
\usepackage{mathptmx}                  % use matching math font
\usepackage{times}                     % we use Times as the main font
\usepackage{cite}
\usepackage{caption}
\usepackage{morefloats}
\usepackage[tight,normalsize,sf,SF]{subfigure}

\usepackage{amsmath,amscd}
\usepackage{amsfonts}
\usepackage[ruled,vlined]{algorithm2e}
\usepackage{enumerate}
\usepackage{multirow}
\usepackage{array}
\newtheorem{corollary}{Corollary}

\newtheorem{proof}{Proof}
\newtheorem{lemma}{Lemma}

\onlineid{0}

\vgtccategory{Research}

\title{Corresponding Supine and Prone Colon Visualization Using Eigenfunction Analysis and Fold Modeling}

\author{Saad Nadeem, Joseph Marino, Xianfeng Gu, and Arie Kaufman, \textit{Fellow, IEEE}}
\authorfooter{
\item
Saad Nadeem, Joseph Marino, Xianfeng Gu, and Arie Kaufman are with the Computer Science Department, Stony Brook University, Stony Brook, NY 11794-2424.
E-mail: \{sanadeem,jmarino,gu,ari\}@cs.stonybrook.edu
}

\abstract{We present a method for registration and visualization of corresponding supine and prone virtual colonoscopy scans based on eigenfunction analysis and fold modeling. In virtual colonoscopy, CT scans are acquired with the patient in two positions, and their registration is desirable so that physicians can corroborate findings between scans. Our algorithm performs this registration efficiently through the use of Fiedler vector representation (the second eigenfunction of the Laplace-Beltrami operator). This representation is employed to first perform global registration of the two colon positions. The registration is then locally refined using the haustral folds, which are automatically segmented using the 3D level sets of the Fiedler vector. The use of Fiedler vectors and the segmented folds presents a precise way of visualizing corresponding regions across datasets and visual modalities. We present multiple methods of visualizing the results, including 2D flattened rendering and the corresponding 3D endoluminal views. The precise fold modeling is used to automatically find a suitable cut for the 2D flattening, which provides a less distorted visualization. Our approach is robust, and we demonstrate its efficiency and efficacy by showing matched views on both the 2D flattened colons and in the 3D endoluminal view. We analytically evaluate the results by measuring the distance between features on the registered colons, and we also assess our fold segmentation against 20 manually labeled datasets. We have compared our results analytically to previous methods, and have found our method to achieve superior results. We also prove the hot spots conjecture for modeling cylindrical topology using Fiedler vector representation, which allows our approach to be used for general cylindrical geometry modeling and feature extraction.
} % end of abstract


\CCScatlist{ % not used in journal version
 \CCScat{K.6.1}{Management of Computing and Information Systems}%
{Project and People Management}{Life Cycle};
 \CCScat{K.7.m}{The Computing Profession}{Miscellaneous}{Ethics}
}

   \teaser{
        \centering
        \footnotesize
        \begin{tabular}{cc}
        \includegraphics[width=0.435\textwidth]{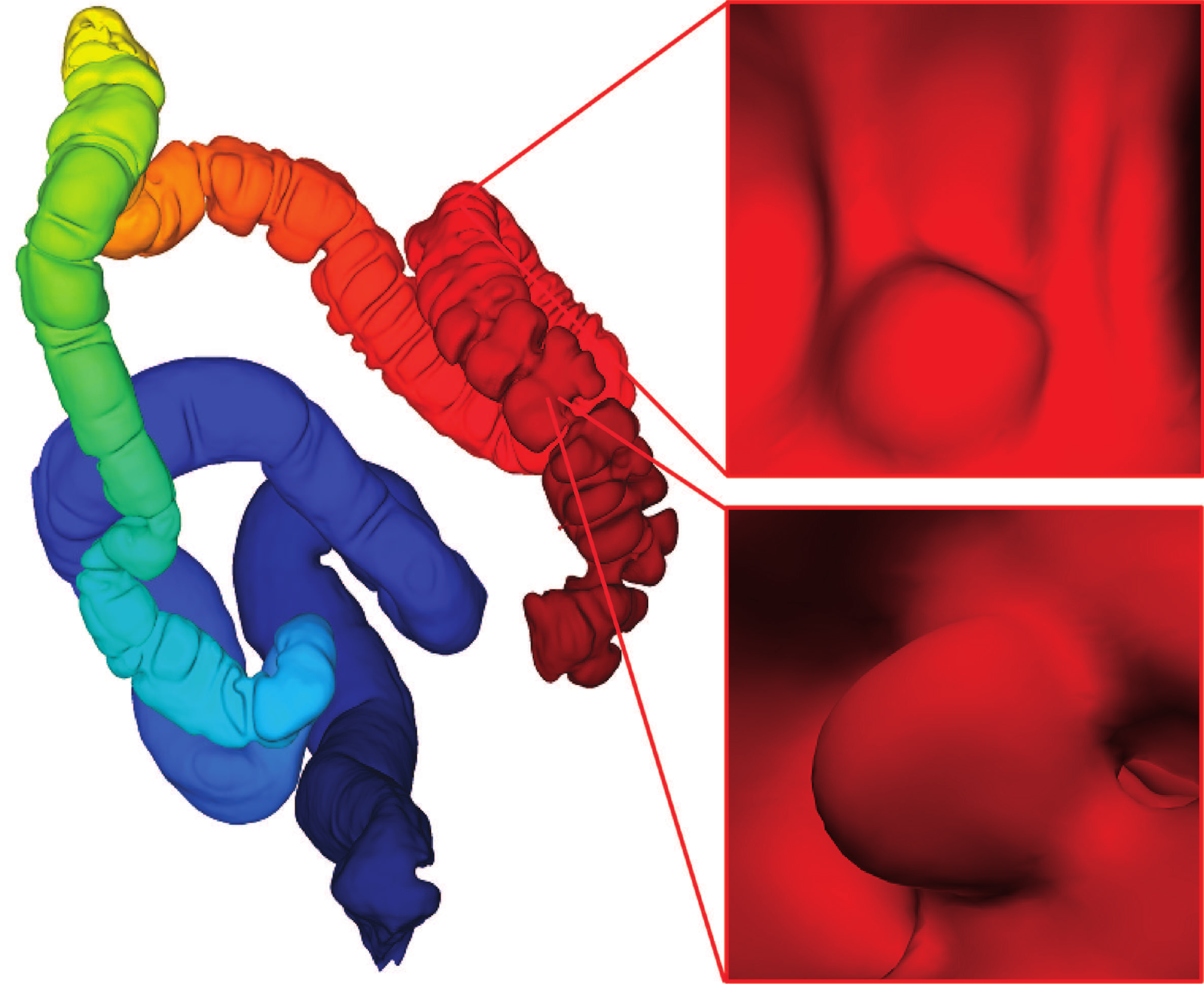}&
        \includegraphics[width=0.435\textwidth]{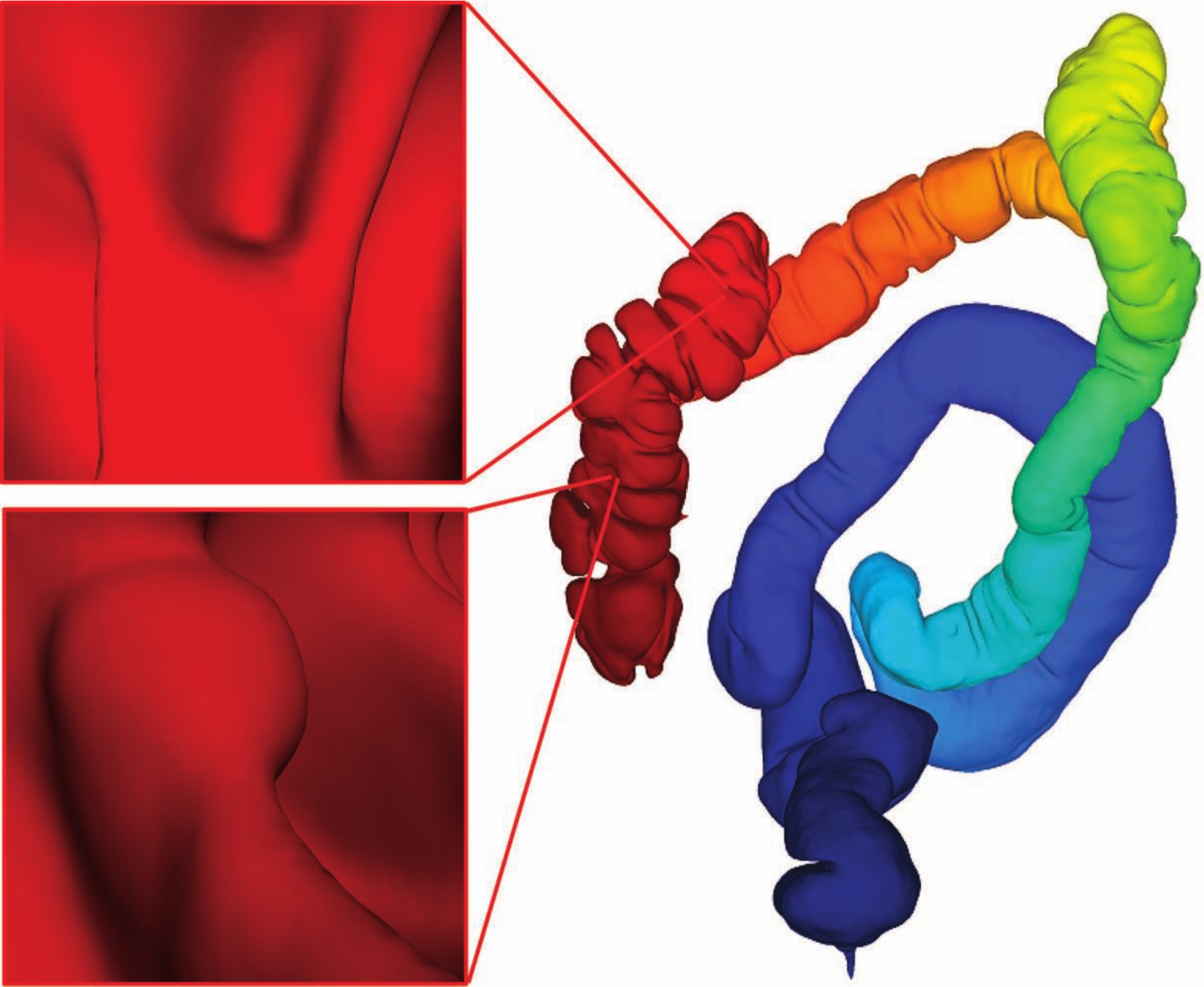}\\
        (a) & (b)
        \end{tabular}
        \vspace{-2.2mm}
        \caption{Corresponding supine and prone colon visualization. (a) A supine colon model with Fiedler vector representation and two polyps. (b) Registered prone model with corresponding polyps correlated in the endoluminal view with those in the supine model.
        \label{fig:teaser}}
  }

\vgtcinsertpkg

\begin{document}


\firstsection{Introduction}

\maketitle

%\section{Introduction}
%\label{sec:intro}
% Using \firstsection

Colorectal cancer (CRC) is the third most frequently diagnosed cancer worldwide, and the fourth leading cause of cancer-related mortality with 700,000 deaths per year worldwide. Optical colonoscopy (OC) is an uncomfortable invasive technique commonly used to screen for CRC. In contrast, virtual colonoscopy (VC) was introduced as a non-invasive procedure for the mass screening of polyps \cite{hong:1997:siggraph}, the precursors to CRC. VC reconstructs a 3D colon model from a computed tomography (CT) scan of a patient's abdomen and then a physician virtually navigates through the colon looking for polyps. The low radiation dosage, the recent assignment of CPT code 74263 for insurance reimbursement, the lack of side effects from anesthesia, and the quick turnaround time for results makes VC a prime method for CRC screening.

\begin{figure*}[t!]
\begin{center}
\includegraphics[width=1\textwidth]{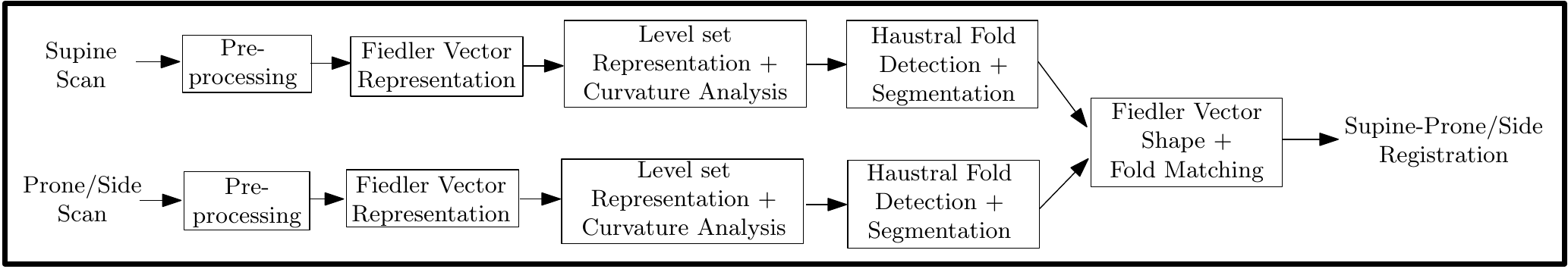}
\end{center}
\vspace{-6.5mm}
\caption{The pipeline for supine-prone/side registration.
\label{fig:pipeline}}
\vspace{-3.6mm}
\end{figure*}

In a typical VC, the patient is scanned twice, supine (facing up) and prone (facing down) or, more recently, in a side position called the lateral recumbent position to avoid artifacts due to the squishing of the colon in the prone position. The two scans improve the sensitivity and specificity of polyp detection and reduce the extent of uninterpretable collapsed or fluid-filled segments. Because the colon changes between the supine and prone (or side) scans, the two colon models need to be registered to match detections. Computer-aided detection (CAD) techniques \cite{hong:2006:tvcg,zhao:2006:tvcg} have been introduced to automatically detect polyps, but they have issues. Registration can aid these CAD techniques by reducing false positives (primarily due to stool and colon folds) when matched across scans. Moreover, the recurrence of cancer due to the incomplete removal of polyps can be tracked via registration of scans captured at different periods. This allows physicians to perform an effective follow-up examination by visualizing the polyp growth over time. Hence, registration provides an effective tool to distinguish between lesions, aid in CAD, and visualize polyp growth over time. In this paper, we focus on registering VC data across multiple patient positions.

The first non-trivial eigenfunction of the Laplace-Beltrami operator is known as the Fiedler vector representation. This representation can be used by itself for complex shape modeling of tubular structures such as the colon. Lai et al. \cite{lai:2010:intra} proposed the use of the Fiedler vector for supine and prone colon registration. However, their registration depends on the accurate detection of important landmarks, based on assumptions which often did not hold in our patient colon datasets, leading to a higher registration distance error.
%, namely splenic and hepatic flexures, once the Fiedler vectors have been computed for the individual supine and prone colons. They assume that local maximum z-coordinates closer to the maximum and minimum Fiedler vector values denote the splenic and hepatic flexures, respectively. However, in real patient colon datasets, this assumption is frequently invalid, leading to an increase in their registration distance error compared to the baseline where registration is done solely using the Fiedler vectors and without taking into account these landmarks.}
Instead, we integrate the Fiedler vector representation with a more accurate feature detection based on haustral fold modeling to achieve better registration and visualization results. We first use the Fiedler vector representation for global registration of the scans. We then use 3D level sets computed based on this representation to segment the folds. Finally, we refine the registration locally via these folds as anatomical references.
%Accurate fold modeling will provide more accurate registration due to anatomical references and effective ways of visualizing corresponding regions across scans. Moreover, we use 3D level sets computed from this representation to segment the colon folds.
An overview of our pipeline is given in Figure \ref{fig:pipeline}. The accurate modeling of the colon folds and the Fiedler vector representation results in automatic cross visualization of consistent endoluminal views across multiple orientations.

Different flattening approaches have been introduced to visualize the complex geometry of medical organs \cite{bartroli:2001,gurijala:2013,haker2000conformal,hong:2006:tvcg,marino:2011,marino:2016,schwartz:1986}. The Fiedler vector representation and fold segmentation also allows us to create more accurate 2D flattening visualizations by automatically computing consistent cuts throughout the colon without crossing the folds. These cuts are then used as an input to a quasi-conformal mapping \cite{zeng2010supine}, which is used to flatten the colon. Previous flattening approaches require manually selecting extrema on the colon and computing geodesic paths to create a cut. This cut, however, crosses folds on sharp bends since this is considered the shortest path, resulting in flattened visualizations with chopped folds, as shown in Figure \ref{fig:automatic_cut}.

%Word or two about traditional registration approaches!!!
There have been other registration approaches which require either using the centerline or flattening approaches which distort texture and require high quality data. Moreover, once the data is flattened to a 2D rectangle, it is difficult to find correspondences between the original data and the flattened data due to the loss of geometry. It is our endeavor here to be able to easily visualize this correspondence and correlate anatomical features across scans in the original and the flattened data.
%The challenges of 2D flattening approaches and how our accurate modeling approach provides for better visualizations by providing best cuts!!!
%In this paper, the registration algorithm works solely on the original 3D data and does not require the data to be flattened to a simpler topology to pose it as a 2D matching problem. We only use flattening to visualize the results of our algorithm and show the correspondence.
Our registration framework makes use of accurate fold models to segment out folds with high precision. These fold segmentations are then used as anatomical references to register the VC data across scans. Moreover, the fold segmentations allow us to automatically find suitable cuts to map the 3D data to simpler 2D topologies and visualize the resultant data more effectively. 
%\section{Contributions}
%Strong visualization angle!!!
Our contributions are as follows:
\begin{itemize}
\itemsep0em
  \item Integrated colon registration framework using Fiedler vector representation for global registration and 3D level set fold segmentation for local refinements.
  \item Accurate fold modeling which provides state-of-the-art detection and segmentation results on colon models.
%, which in turn provides more accurate anatomical references for the eventual registration.
  \item Accurate correspondence visualization between 3D and 2D flattening visualizations due to the Fiedler vector representation and accurate fold segmentation.
  \item Extraction of a consistent cut along the colon for visually accurate colon flattening.
%  \item The theoretical proof and mathematical formulation of the Fiedler vector representation and its application to colon registration.
  \item The theoretical proof of the hot spots conjecture for modeling general cylindrical topology using Fiedler vector representation.

%Due to our accurate registration, we can retrieve consistent endoluminal views automatically for a given point on the supine/prone colon on the corresponding supine/prone colon. These aspects allow for more accurate polyp localization on both 3D and 2D flattened models with less distortion. Since the folds are well-segmented, we can also localize polyps using the surrounding fold labels.
\end{itemize}

The paper is organized as follows. In the next section, we discuss related work, followed by a brief theoretical background of Fiedler vector representation in our context in Section \ref{sec:theory}. This provides the basis for our colon registration algorithm, which we outline in Section \ref{sec:algorithm}. In Section \ref{sec:vis}, we highlight the supine and prone correspondence and more effective flattening visualizations based on our algorithms. Finally, we evaluate our algorithm and show additional results in Section \ref{sec:experiments}, followed by the conclusion and future work in Section \ref{sec:conc}.

%\vspace{-1mm}
\section{Related Work}
\label{sec:related}

There have been several works in the colon registration domain based on the identification of landmarks and centerline matching across supine/prone/side scans.
A basic method applies linear stretching and shrinking operations to the
centerline, where local extrema are matched and used to drive the
deformations~\cite{acar:2001:registration,acar:2001:embs}. Correlating individual points along the centerline through the use of dynamic programming has also been
suggested~\cite{devries:2006:bjr,li:2004:medphys,nain:2002:miccai}.

More recently, the taeniae coli (three bands of smooth muscle along the colon surface) have been used as features which can be correlated between the two
scans~\cite{huang:2005:vis}. This relies on a
manual identification of one of the three taeniae coli, and then an
automatic algorithm repeats the line representing the identified
taenia coli at equal distances. Further progress has been made where
the haustral folds and the points between them can be automatically
detected, and the taeniae coli are identified by connecting these
points~\cite{umemoto:2008:spie}. However, this method is only
feasible on the ascending and transverse portions of the colon.

Deformation fields have also been suggested for use in supine-prone
registration. Motion vectors can be identified for matched
centerline regions, interpolated for non-matched regions, and then
propagated to the entire volume~\cite{suh:2009:jcat}. It has also
been proposed to use a free-form deformation grid to model the
possible changes in the colon shape from supine to
prone~\cite{plishker:2008:miccai}.

Conformal mapping has been successfully used for many medical applications, including a brain cortex surface morphology study \cite{haker2000conformal} and colonic polyp detection \cite{hong:2006:tvcg}. Quasi-conformal mapping was used to map supine and prone colon datasets to a rectangular plane and convert the 3D registration problem into a 2D image matching problem \cite{zeng2010supine}. This, however, required landmarks to be extracted in order to divide the colon into different well-defined segments. Moreover, the whole registration pipeline was dependent on flattening these segments by tenaie coli extraction and then using graph matching. The flattening was computationally expensive, and precise fold segmentation could have improved the accuracy of the graph matching algorithm.

\begin{figure}[t!]
\begin{center}
\begin{tabular}{cc}
\includegraphics[width=0.175\textwidth]{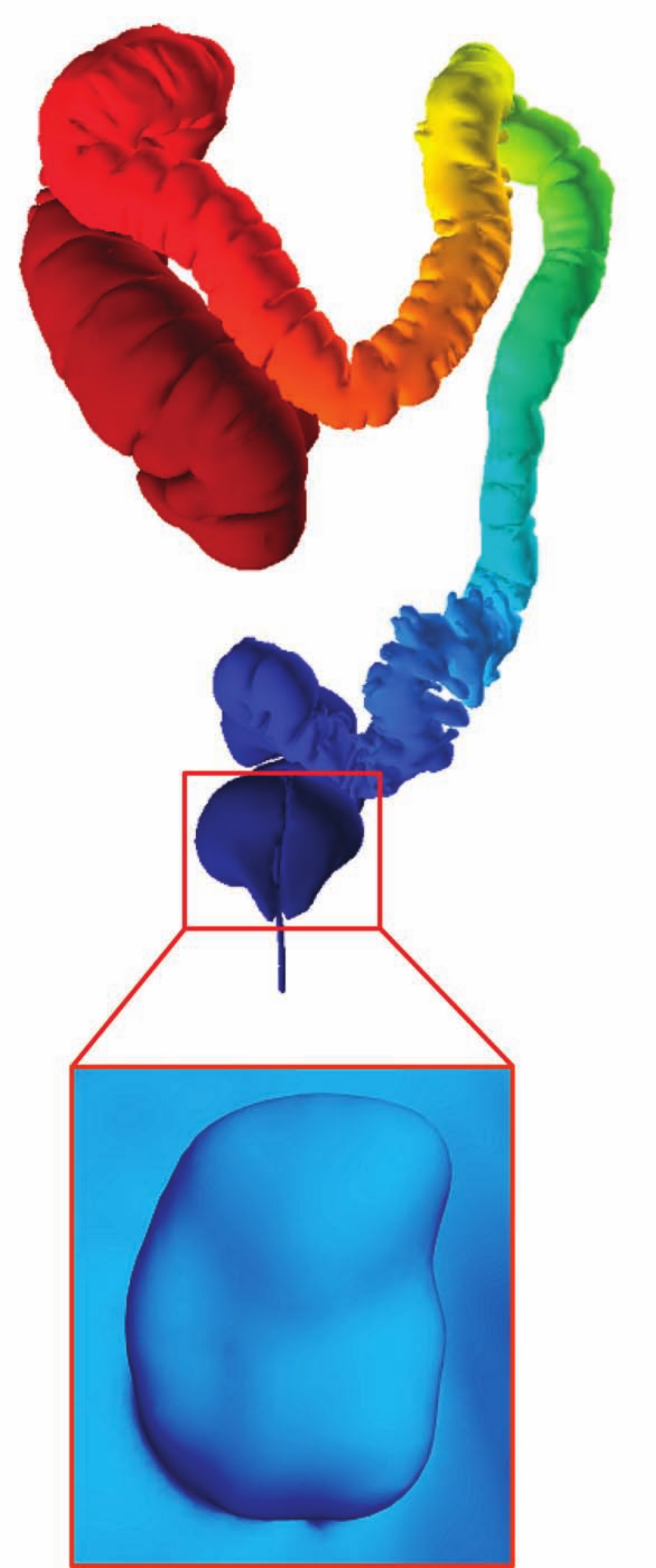}&
\includegraphics[width=0.175\textwidth]{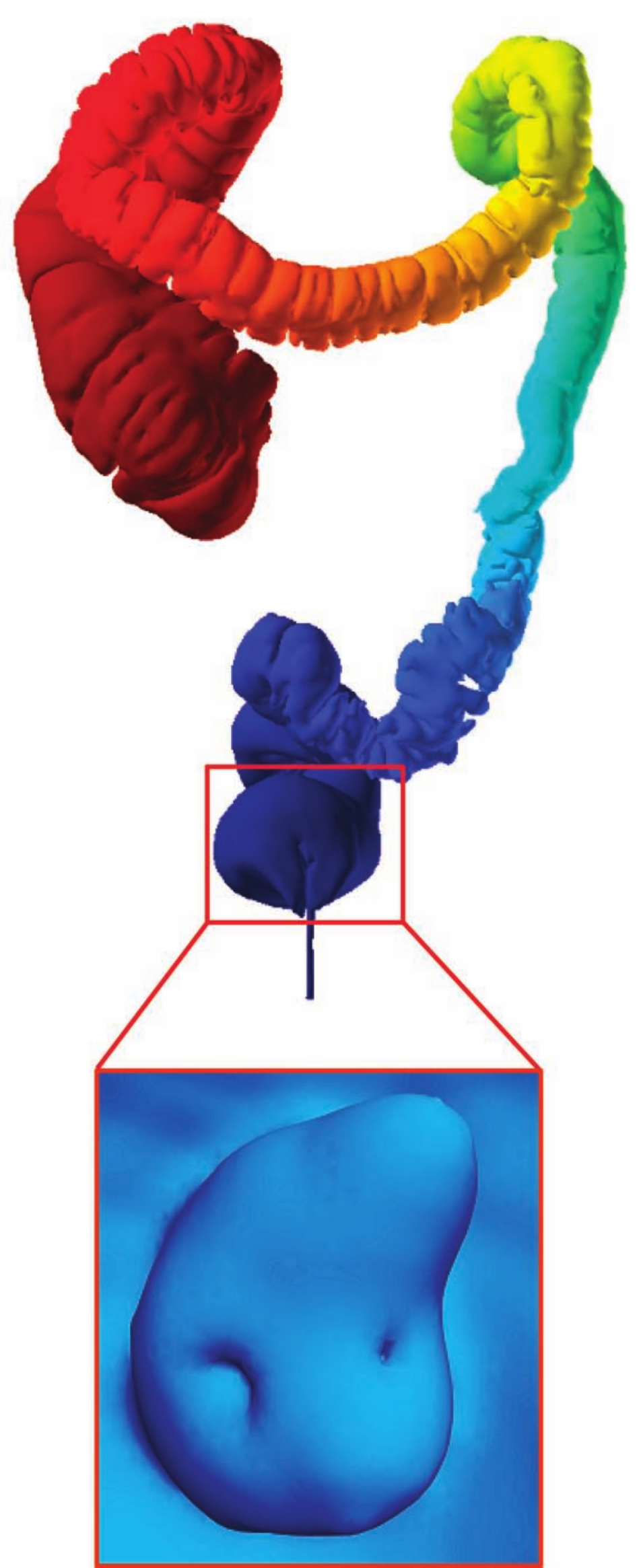}\\
(a) Supine & (b) Side
\end{tabular}
\end{center}
\vspace{-6mm}
\caption{Fiedler vector representation for (a) supine and (b) side colon models with corresponding polyps detected in the rectum.
\label{fig:supineside}}
%\vspace{-2mm}
\end{figure}

There has been minimal use of folds as anatomical references for registration due to the imprecise modeling of folds, thus leading to poor fold detection and segmentation results.  Effective segmentation of the colon folds is necessary for detecting polyps on the folds \cite{yao:2007:detecion}, extracting teniae coli \cite{huang:2005:vis,zhu:2011:automatic}, and performing supine-prone registration \cite{acar:2001:embs,acar:2001:registration,wang:2009:registration,zeng2010supine}. Most of the previous works have focused on colon fold detection \cite{chowdhury:2010:colonic,huang:2005:vis,wei:2010:teniae}, disregarding the precise segmentation or delineation of the boundaries of the fold. A fold boundary modeling and segmentation algorithm was introduced, and the segmented-area ratio ($SAR$) metric was also introduced to measure the accuracy of the segmentation \cite{zhu2013haustral}. In this work, we model and segment the folds (using Fiedler vector representation) on the supine-prone or supine-side VC dataset pairs which in turn are used as anatomical references to register these pairs.
%However, the segmentation algorithm assumed that the VC datasets did not have any severe under/over colon distention. In our work, we do not make such assumptions.

Fiedler vector representation has been used in various applications including mesh processing \cite{gebal:2009}, mesh parameterization \cite{nadeem:2016}, and shape characterization \cite{boscaini:2015shape}. Lai et al. \cite{lai:2010:intra} have proposed the use of Fiedler vector representation for supine and prone colon registration. Based on the Fiedler vector computation for supine and prone colons, they detect the hepatic and splenic flexures, and register the supine and prone colons piecewise using these detected landmarks. They assume that the splenic and hepatic flexures are denoted by local maximum z-coordinates nearest to the maximum and minimum Fiedler vector values, respectively. This assumption, however, is frequently invalid in real patient colon datasets which we have tested on and can result in higher distance registration error compared to the registration done solely based on the Fiedler vector without considering the landmarks. Moreover, their algorithm does not take into account the Fiedler vector flips that can occur between the corresponding supine and prone colons which could lead to erroneous registration results (e.g. the rectum-splenic segment being registered to the cecum-hepatic segment).

\begin{figure}[t!]
\begin{center}
\begin{tabular}{c}
\includegraphics[width=0.38\textwidth]{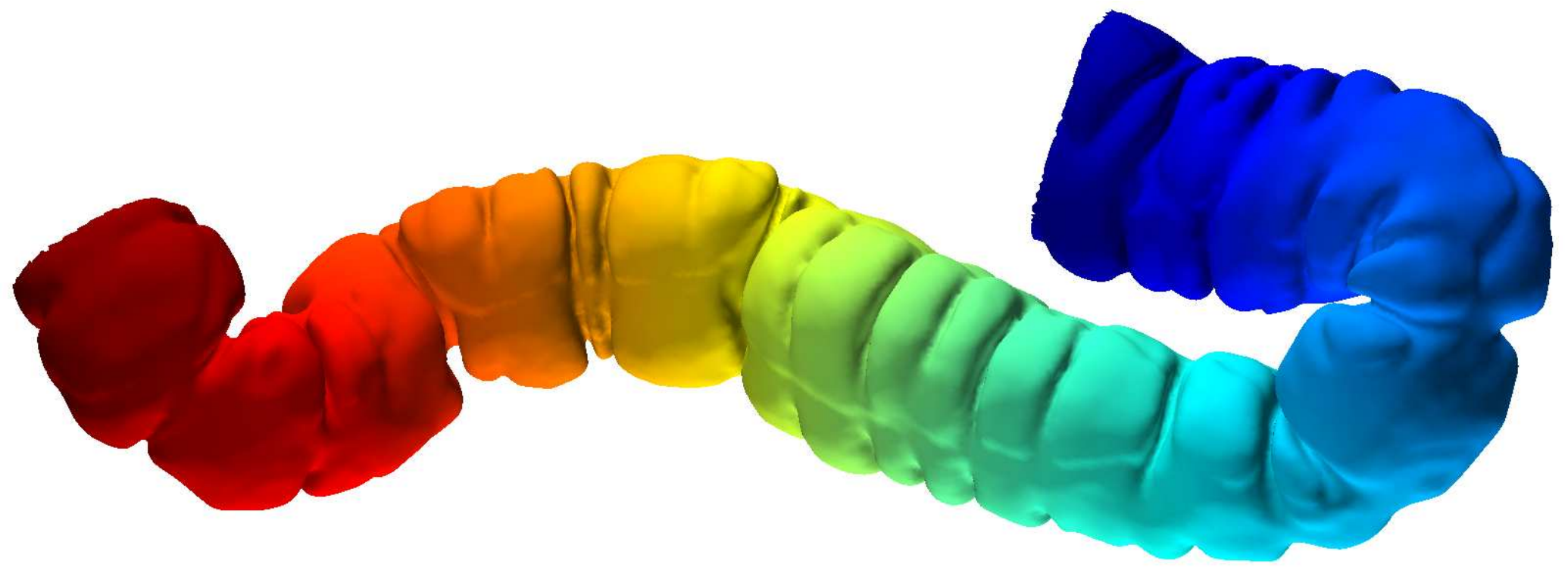}\vspace{-0.5mm}\\
(a) Fiedler vector on colon segment\\
\includegraphics[width=0.38\textwidth]{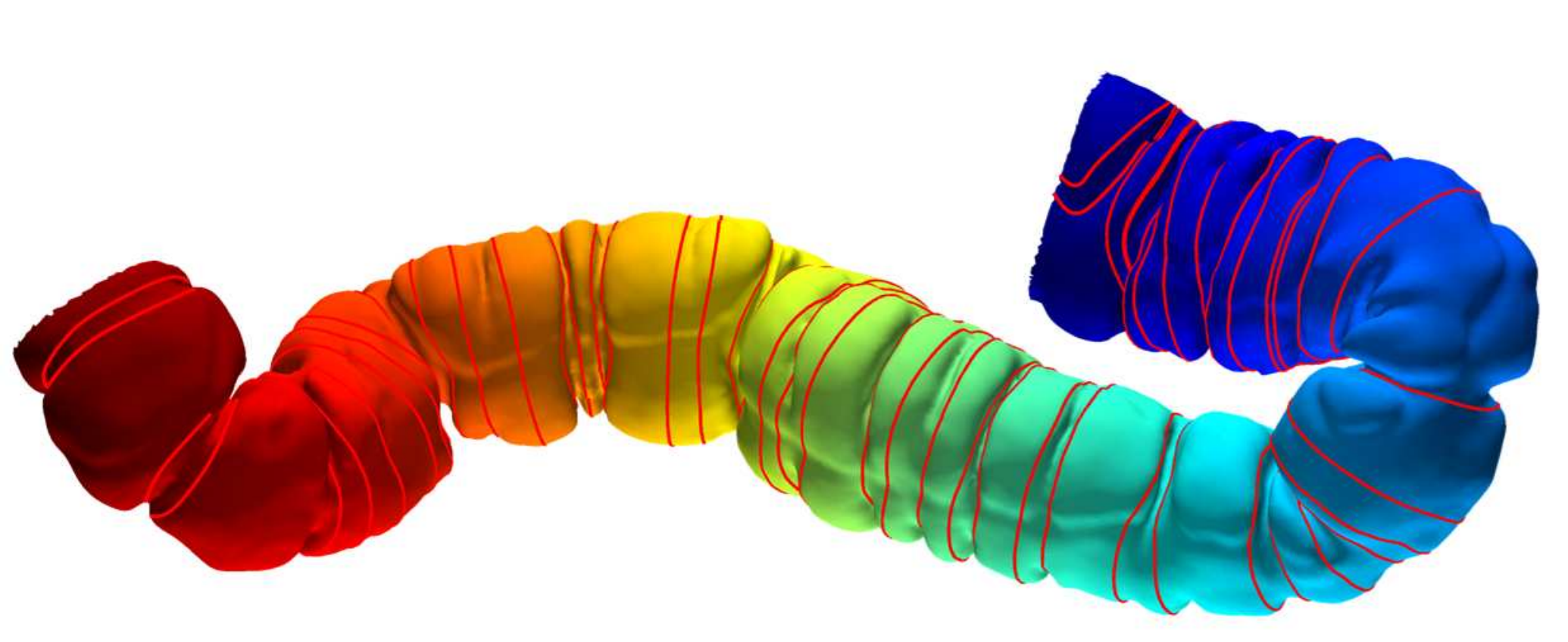}\vspace{-1.5mm}\\
(b) Fiedler vector level sets on colon segment\vspace{1mm}\\
\includegraphics[width=0.38\textwidth]{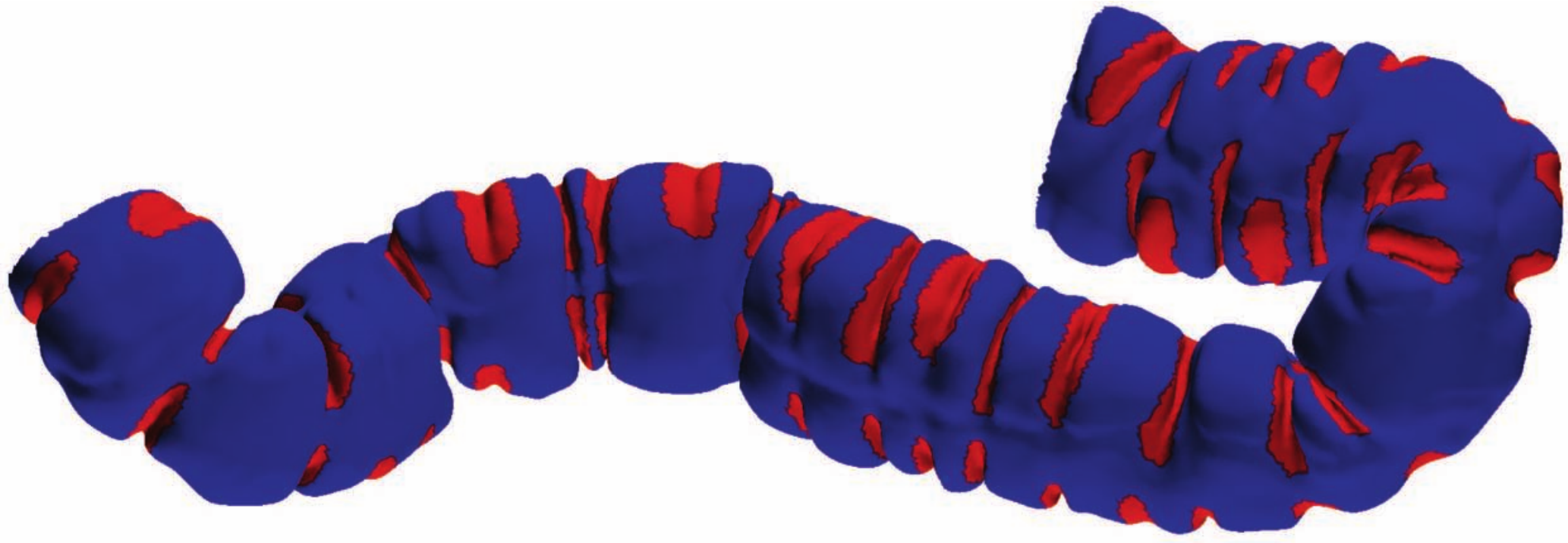}\\
(c) Segmented folds on colon segment\vspace{1mm}\\
\includegraphics[width=0.43\textwidth]{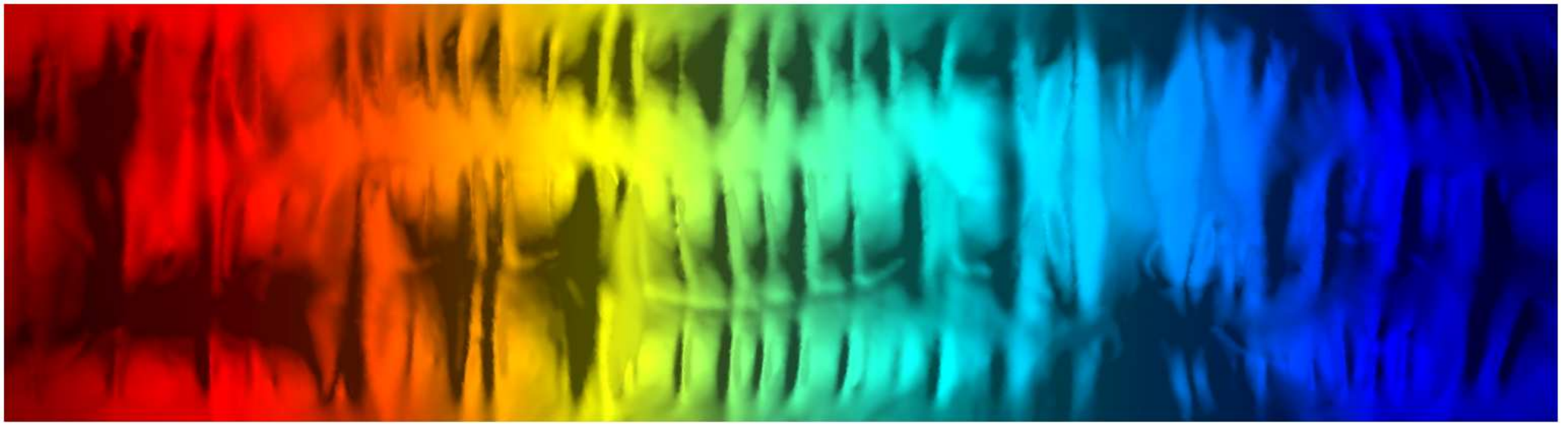}\\
(d) Fiedler vector on flattened colon segment\vspace{1mm}\\
\includegraphics[width=0.43\textwidth]{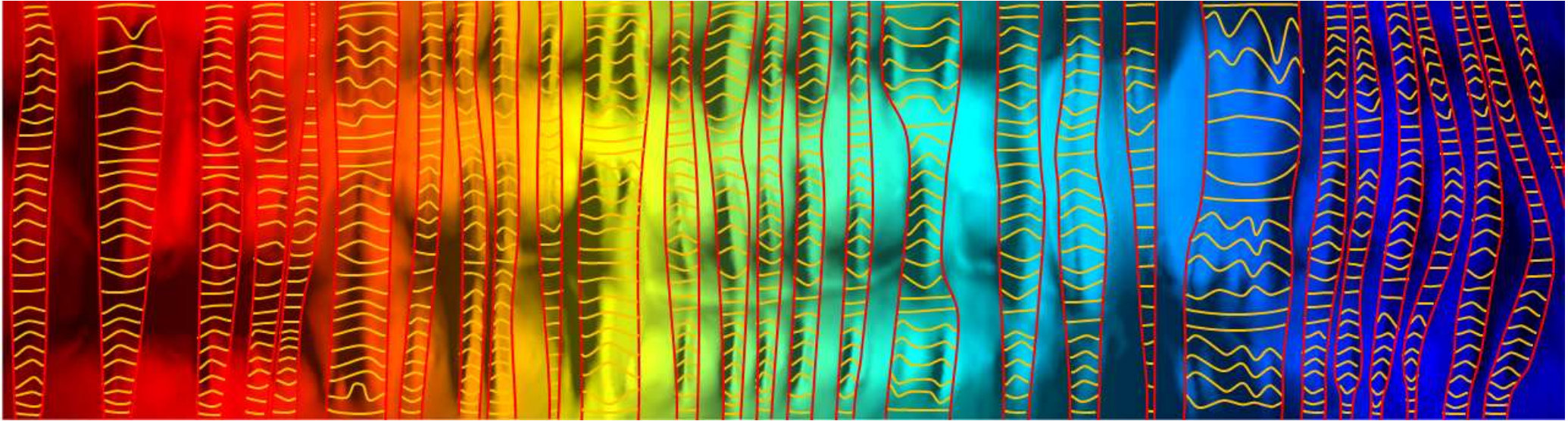}\\
(e) Fiedler vector level sets and cross curves on flattened segment\vspace{1mm}\\
\includegraphics[width=0.43\textwidth]{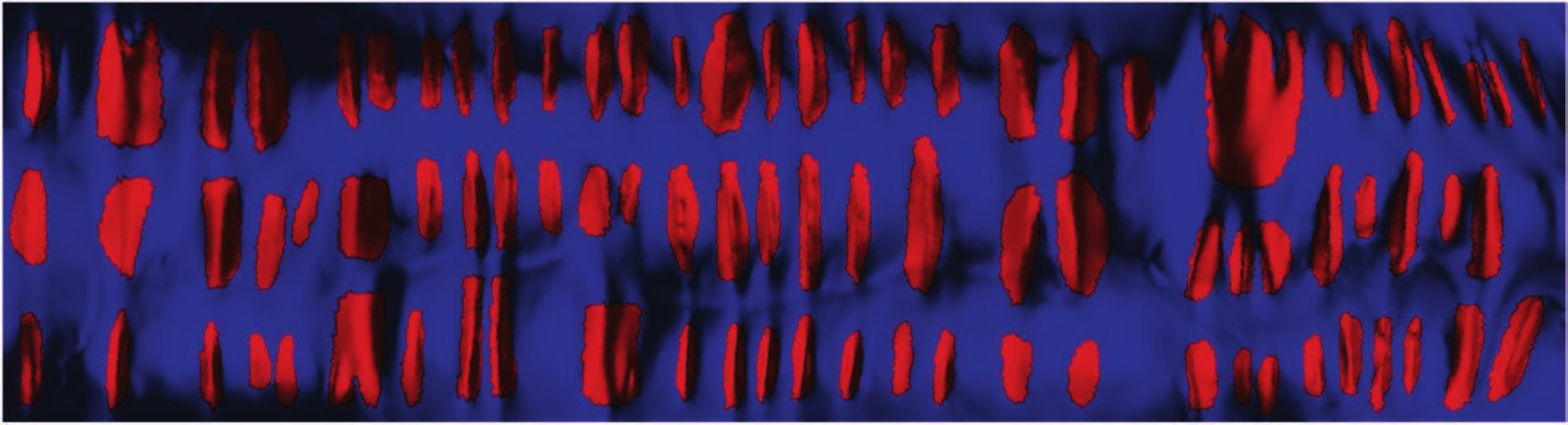}\\
(f) Segmented folds on flattened colon segment\\
\end{tabular}
\end{center}
\vspace{-5.2mm}
\caption{Haustral fold detection and segmentation pipeline. The computed Fiedler vector is shown for a colon segment in the (a) 3D view and (d) corresponding flattened view. The 3D level sets, illustrated with red lines, are computed to detect fold contours and are shown in the (c) 3D view and together with cross curves, illustrated in yellow lines, in the (e) flattened view.  Using these, the haustral folds are segmented, as shown in the (c) 3D view and (f) corresponding flattened view.
\label{fig:supine0}}
\vspace{-1mm}
\end{figure}

In contrast, we use the Fiedler vector representation to model the human colon, register the supine/prone/side models in 3D, and segment folds using level sets based on this representation. The Fiedler vector representation is used to register the supine/prone/side datasets globally and then the segmented folds are used as anatomical references to locally refine the registration results. The fold segmentation based on the level sets is efficient and the overall registration process does not depend on the computation of the centerline, the teniae coli, or the flexures. Our use of the Fiedler vector representation for fold segmentation is closely related to spectral image \cite{liu:2004:segmentation} and shape segmentation \cite{shi:2000:normalized}. The colon registration in 3D, in essence, is similar to shape matching using functional maps \cite{kovnatsky:2013:coupled,ovsjanikov:2012:functional}. 
\section{Theoretical Background}
\label{sec:theory}
In this section, we briefly introduce the theoretical background necessary for this work, which allows us to compute extrema and model the complex geometry of the colon. The hot spots conjecture \cite{chung:2011} states that the minimum and maximum of the Fiedler vector are the two points with the greatest geodesic distance on the surface. We prove in Lemma \ref{lem:hot_spot} that this conjecture holds for a topological cylinder.

%\subsection{Surface Registration with Folds}
\begin{figure*}[ht]
\begin{center}
\begin{tabular}{ccc}
\includegraphics[width=0.3\textwidth,height=0.20\textwidth]{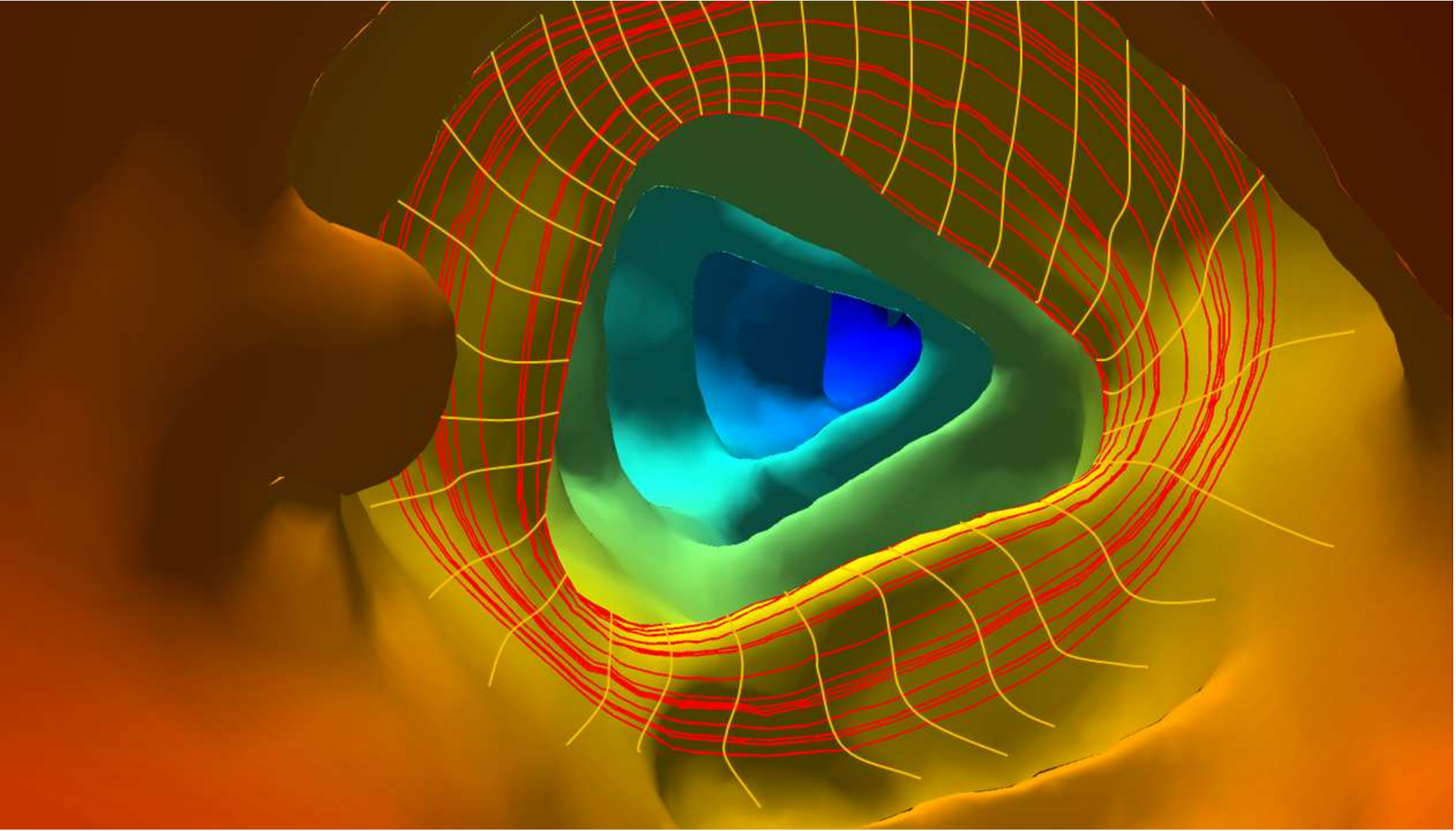} &
\includegraphics[width=0.3\textwidth,height=0.20\textwidth]{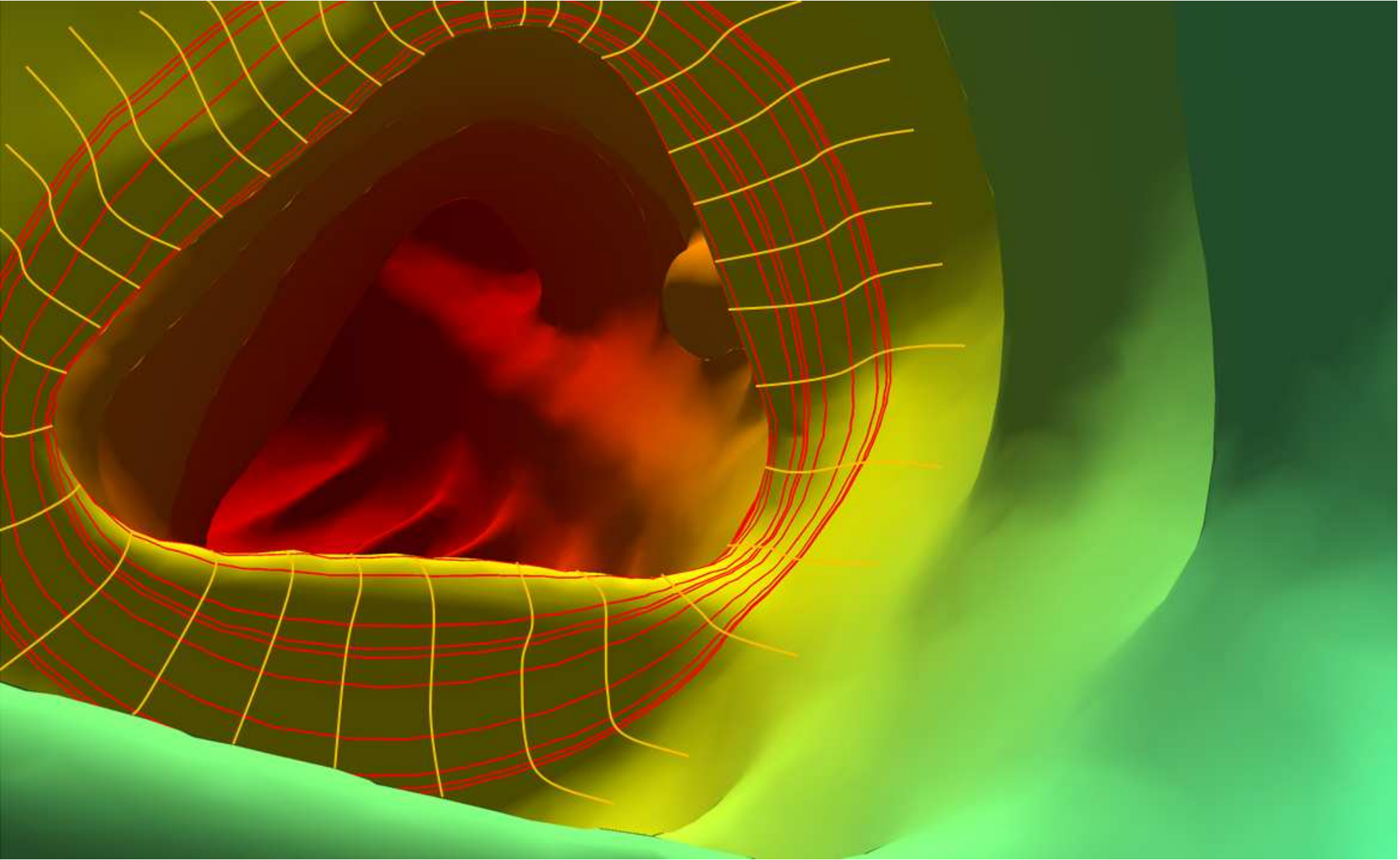} &
\includegraphics[width=0.3\textwidth,height=0.20\textwidth]{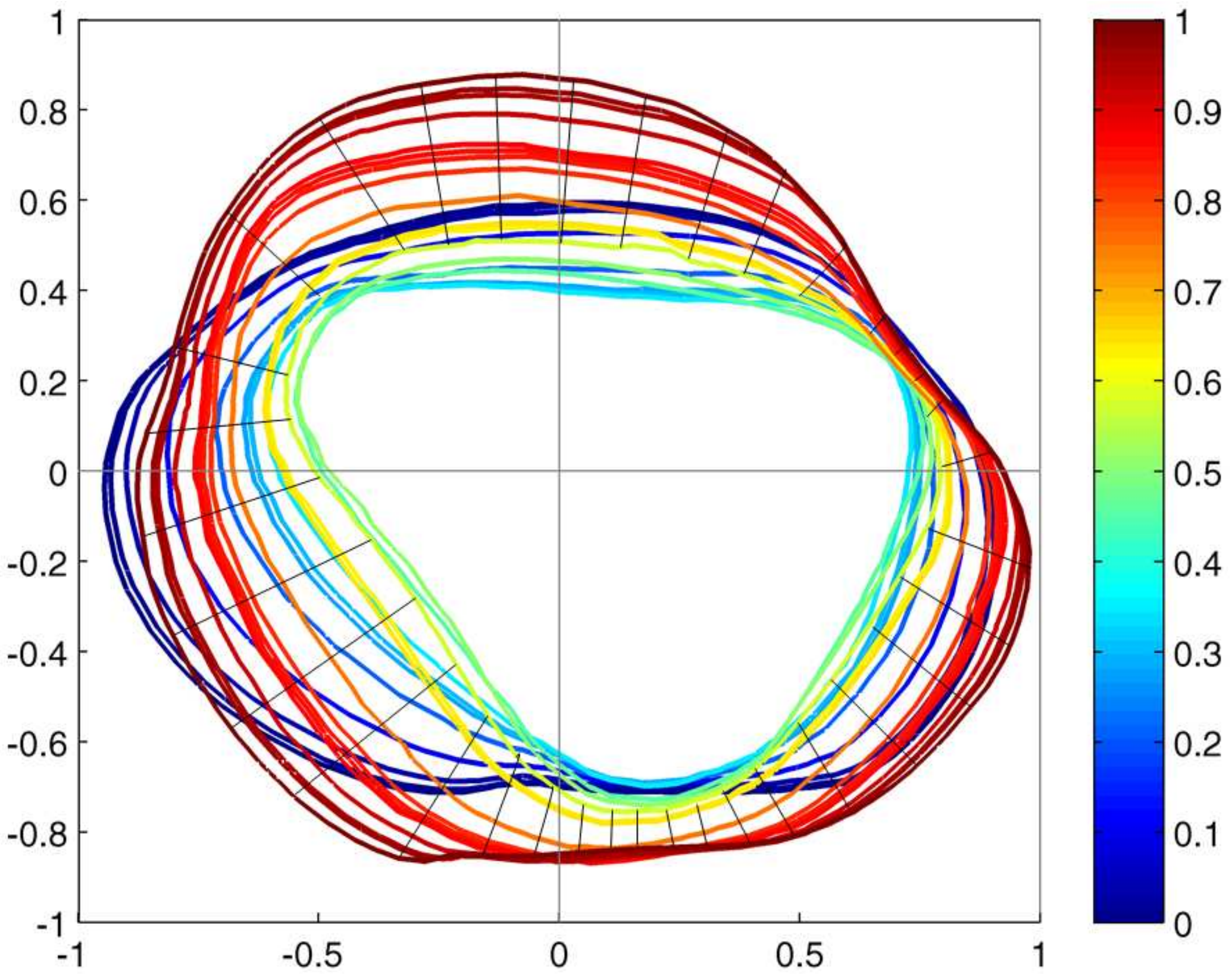}\\
(a) & (b) & (c)\\
\end{tabular}
\end{center}
\vspace{-6.5mm}
\caption{Haustral fold detection. (a) and (b) show two endoluminal views of the same colon segment with a 3D level set bundle (red loops) and corresponding cross curves (yellow lines). (c) Level sets projected on planes and normalized in 2D with the black lines representing the cross curves from an extreme level set in a bundle to the middle level set contour.
\label{fig:fold_planes}}
\vspace{-0.7mm}
\end{figure*}

Suppose $S$ is a surface embedded in the three-dimensional Euclidean space $\mathbb{R}^3$, and the induced Euclidean Riemannian metric is denoted as $\mathbf{g}$. One can choose a special type of local coordinates $(x,y)$, called \emph{isothermal coordinates}, such that the metric has a canonical form $\mathbf{g} = e^{2\lambda(x,y)}(dx^2+dy^2)$,
where the function $\lambda:S\to\mathbb{R}$ is called the \emph{conformal factor}. The \emph{Laplace Beltrami} operator induced by the metric is given by $\Delta_{\mathbf{g}} = \frac{1}{e^{2\lambda}} \left( \frac{\partial^2}{\partial x^2} + \frac{\partial^2}{\partial y^2} \right)$.
The \emph{Gaussian curvature} of the surface is given by $K(x,y) = -\Delta_\mathbf{g} \lambda(x,y)$.
An \emph{eigenfunction} of the Laplace-Beltrami operator is given by $\Delta_{\mathbf{g}} \phi_k  = \lambda_k \phi_k$,
where the \emph{eigenvalue} $\lambda_k \ge 0 $ is a non-negative real number. The Laplace-Beltrami operator has infinitely many eigenvalues and eigenfunctions. The eigenvalues are all sorted in ascending order
$0 = \lambda_0 \le \lambda_1 \le \lambda_2 \cdots \lambda_k \cdots$.
The eigenfunctions form an ortho-normal basis of the functional space of the surface, $\langle \varphi_i, \varphi_j \rangle = \int_S \varphi_i(p) \varphi_j(p) dA_\mathbf{g}=\delta_i^j$. The first eigenfunction $\varphi_0$ is a constant function. The second eigenfunction $\varphi_1$ is called \emph{Fiedler}'s function.

The \emph{heat kernel} of the surface is $K(p,q,t) = \sum_{k=0}^\infty e^{-\lambda_k t} \varphi_k(p)\varphi_k(q)$. The \emph{heat diffusion} equation on the surface is
\begin{equation}
    \frac{du(p,t)}{dt} = \Delta_\mathbf{g} u(p,t),
    \label{eqn:heat_equation}
\end{equation}
and the solution to the heat equation is given by the heat kernel $u(p,t) = \int_{S} K(p,q,t) u(q,0) dA_\mathbf{g}$. Another equivalent way to represent the solution to the heat equation is
\begin{equation}
    u(p,t) = \sum_{k=0}^\infty  \tau_k e^{-\lambda_k t} \varphi_k(p) ,
    \label{eqn:solution}
\end{equation}
where
\[
    \tau_k = \langle \varphi_k, u(p,0)\rangle = \int_S u(q,0) \varphi_k(q) dA_\mathbf{g}.
\]
When $t$ goes to infinity, the right hand side of Equation \ref{eqn:heat_equation} goes to $0$, which means that the temperature function $u(p,t)$ becomes harmonic.

%\paragraph{Hot Spot Property}

\begin{lemma}\vspace{-2mm} Suppose $u:S\to\mathbb{R}$ is a harmonic function, then for any interior point $p\in S$,
\begin{equation}
    u(p) = \frac{1}{2\pi} \oint_\gamma u(q) dq,
    \label{eqn:mean_value}
\end{equation}
where $\gamma$ is a small circle surrounding $p$.
\end{lemma}

\begin{proof}\vspace{-2mm} If we choose isothermal coordinates $(x,y)$, then $u(x,y)$ is a harmonic function and $\Delta u =0$. We construct the conjugate harmonic function $v:S\to\mathbb{R}$, such that $u_x = v_y$ and $u_y = -v_x$, then the complex function $F:S\to\mathbb{C}$ is a holomorphic function, where $z=x+iy$. Note that $F$ is constructed using $u$ as a real part and $v$ as an imaginary part. By Cauchy's formula, we obtain
\[
    F(z) = \frac{1}{2\pi i} \oint_{\gamma} \frac{F(w)}{w-z} dw,
\]
where $w\in \gamma$, $w = z + \varepsilon e^{i\theta}$. Comparing the real and imaginary parts, we obtain Equation \ref{eqn:mean_value}.
\end{proof}

\begin{corollary}\vspace{-2mm} A metric surface $(S,\mathbf{g})$ has boundaries $\partial S$.  Suppose $u$ is a harmonic function on $S$, then the maximal and minimal points of $u$ are on $\partial S$.
\end{corollary}

\begin{proof}\vspace{-2mm} Equation \ref{eqn:mean_value} means that the value of an interior point $u(p)$ is the mean of the values of its neighboring points on the small circle $\gamma$. Therefore, the maximal and minimal points cannot be in the interior of $S$ and hence, these points must be on the boundaries $\partial S$.
\end{proof}

\begin{lemma}\vspace{-2mm} Suppose the metric surface $(S,\mathbf{g})$ is a topological cylinder, then the maximal and minimal points of its Fiedler vector $\varphi_1$ are on the boundaries of the surface $\partial S$.
\label{lem:hot_spot}
\end{lemma}

\begin{proof}\vspace{-2mm}
Suppose the solution to Equation \ref{eqn:heat_equation} is $u(p,t)$. When $t$ goes to infinity, $u(p,t)$ converges to a harmonic function, and therefore the maximal and minimal points of $u(p,t)$ are close to the boundaries $\partial S$.

From Equation \ref{eqn:solution}, when $t$ becomes large enough, the high order eigenfunctions go to $0$ much faster. Therefore, the behavior of $u(p,t)$ is mainly controlled by the first two terms,
\[
    u(p,t)\sim \tau_0 \varphi_0(p) + \tau_1 e^{-\lambda_1 t} \varphi_1(p),
\]
where $\varphi_0$ is constant. Therefore, the maximal and minimal points of $u(p,t)$ are approximated by those of $\varphi_1(p)$, and thus the maximal and minimal points of $\varphi_1(p)$ approach the surface boundaries $\partial S$.
\end{proof}

\paragraph{Level Sets of Fiedler's Function} Suppose $\varphi:(S,\mathbf{g})\to (T,\mathbf{h})$ is a diffeomorphic mapping between two metric surfaces. The local coordinates of $S$ are $(x,y)$ and those of $T$ are $(u,v)$. The \emph{Jacobian} matrix of the mapping $\varphi$ is $J_\varphi$. The \emph{pull back} metric induced by the mapping $\varphi$ is defined by $\varphi^*\mathbf{h} = J_\varphi^T \mathbf{h} J_\varphi$. The mapping $\varphi$ is \emph{conformal} if there is a function $\mu:S\to\mathbb{R}$,
\begin{equation}
    \varphi^*\mathbf{h} = e^{2\mu} \mathbf{g}.
    \label{eqn:conformal_factor}
\end{equation}
According to conformal geometry, if $(S,\mathbf{g})$ is a topological cylinder, then there is a conformal mapping which maps the surface onto a flat cylinder $\mathcal{C}$, $\mathcal{C} = \{(\cos\theta,\sin\theta,h)| 0\le \theta < 2\pi, 0\le h \le H\}$.
The second eigenfunction of $\mathcal{C}$ is $\varphi_1(\theta,h) = e^{h}$. The level sets of the second eigenfunction are circles with constant height $h=const$.

Suppose a topological cylindrical surface $(S,\mathbf{g})$ is conformally mapped onto a flat cylinder $\mathcal{C}$.  Furthermore, the mapping is near-isometric, namely, the conformal factor $\mu$ in Equation~\ref{eqn:conformal_factor} is close to $0$. Then the level sets of the second eigenfunction of $S$ are similar to those on $\mathcal{C}$. In our current work, we use the level sets of the second eigenfunction to locate the folds on the colon surface.

\section{Computational Algorithm}
\label{sec:algorithm}

\begin{figure*}[ht!]
\begin{center}
\begin{tabular}{ccccc}
\includegraphics[width=0.17\textwidth,angle=180, origin=c]{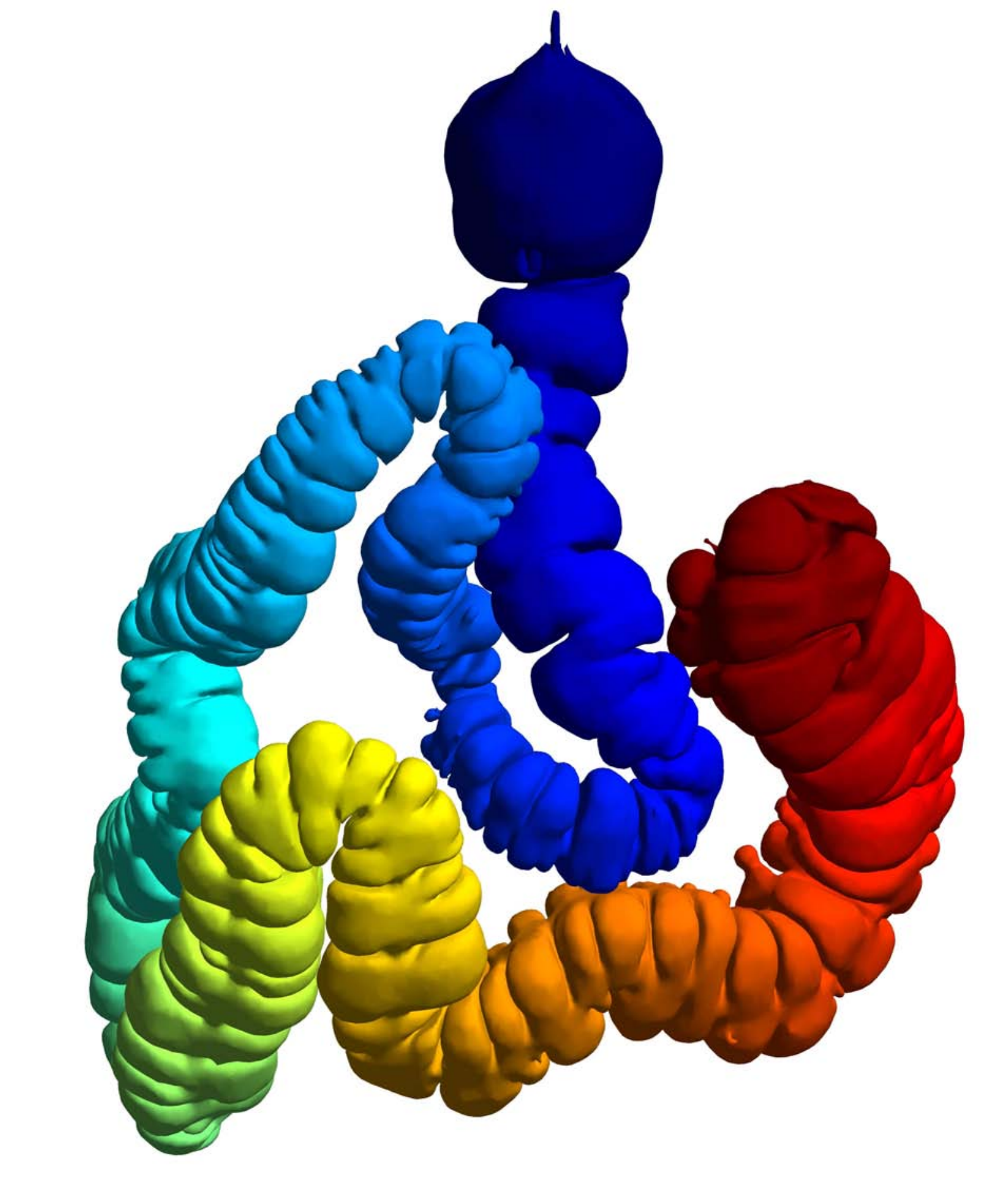} &
\includegraphics[width=0.17\textwidth,angle=180, origin=c]{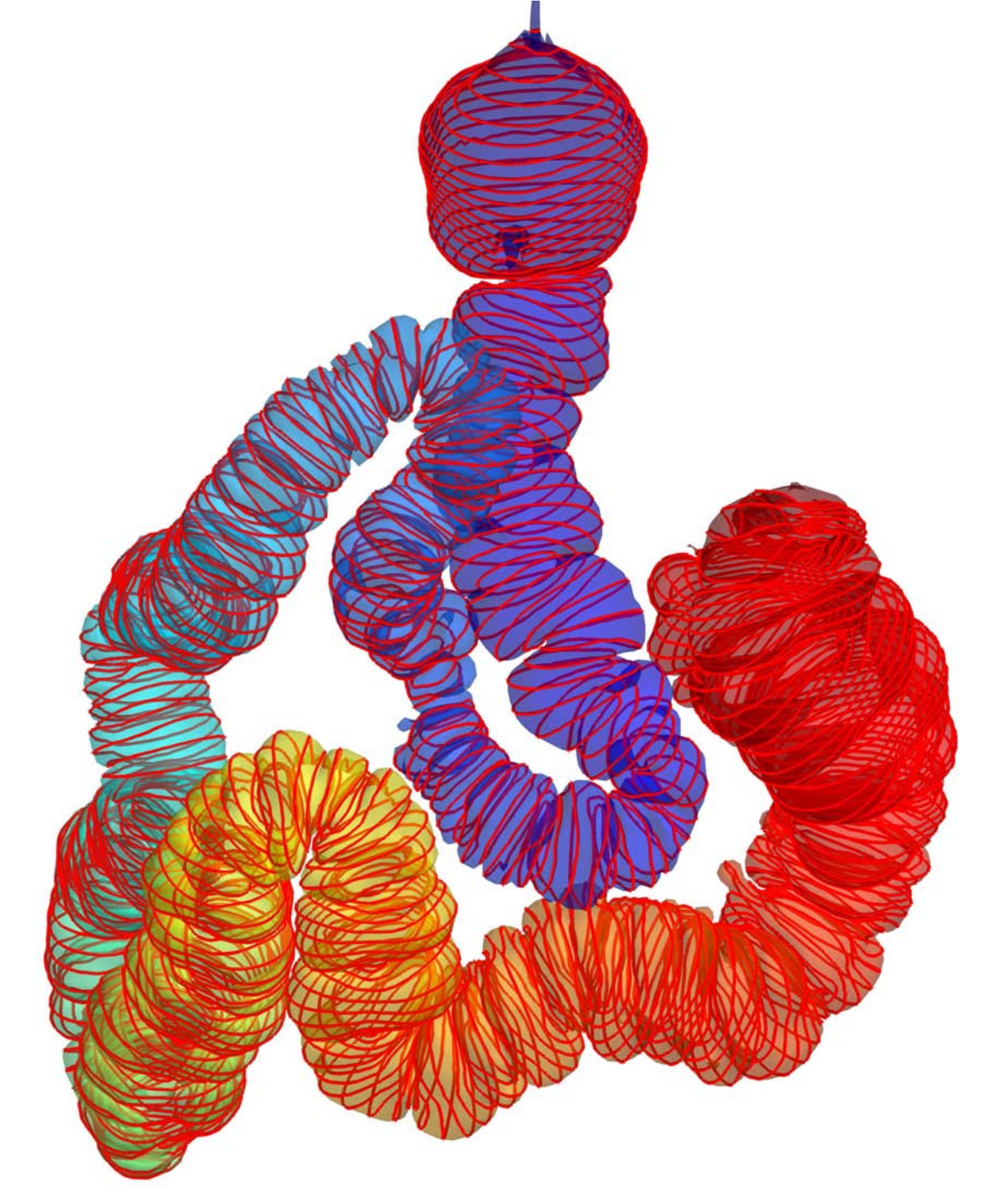} &
\includegraphics[width=0.16\textwidth,angle=180, origin=c]{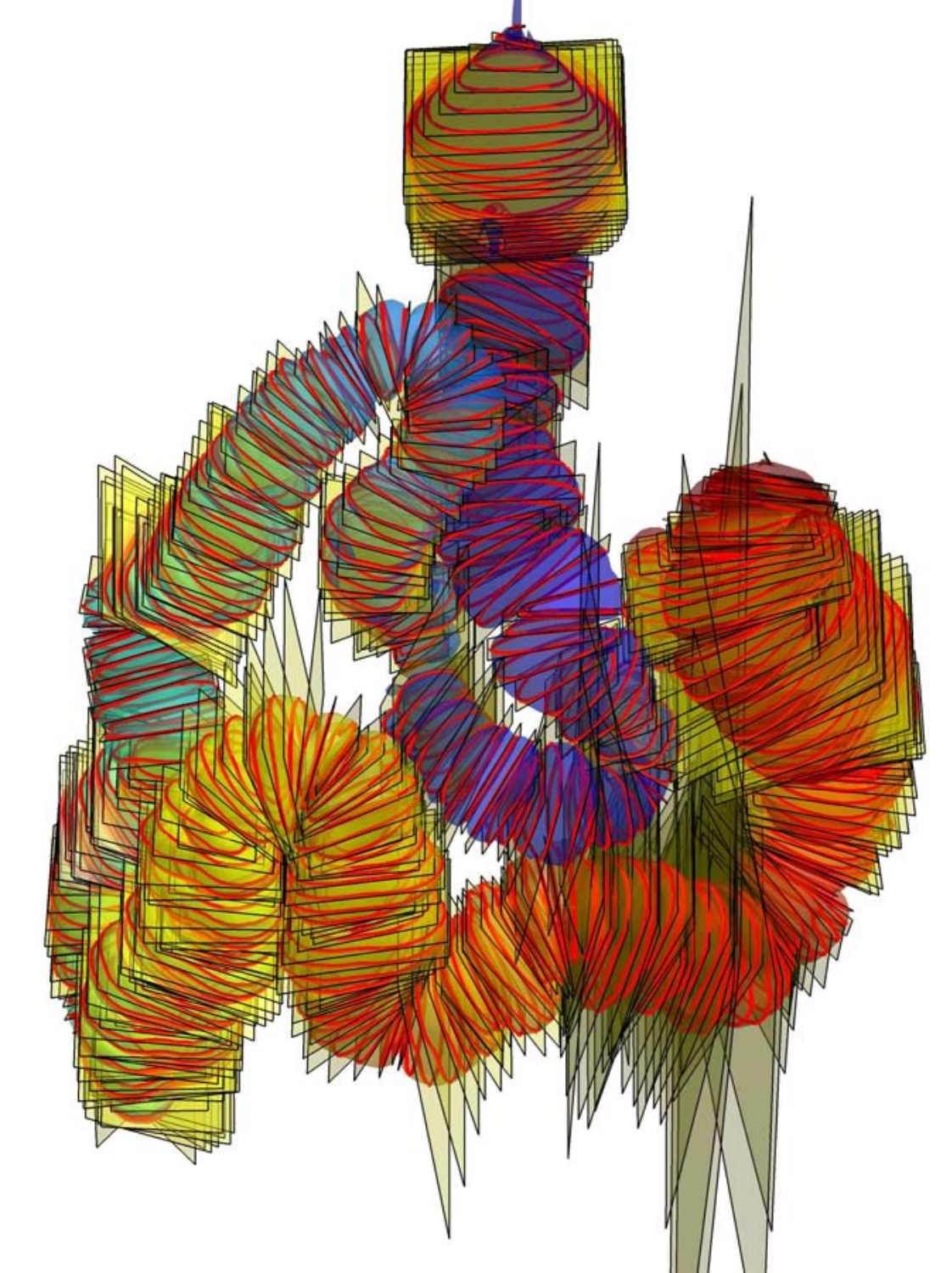} &
\includegraphics[width=0.18\textwidth,angle=180, origin=c]{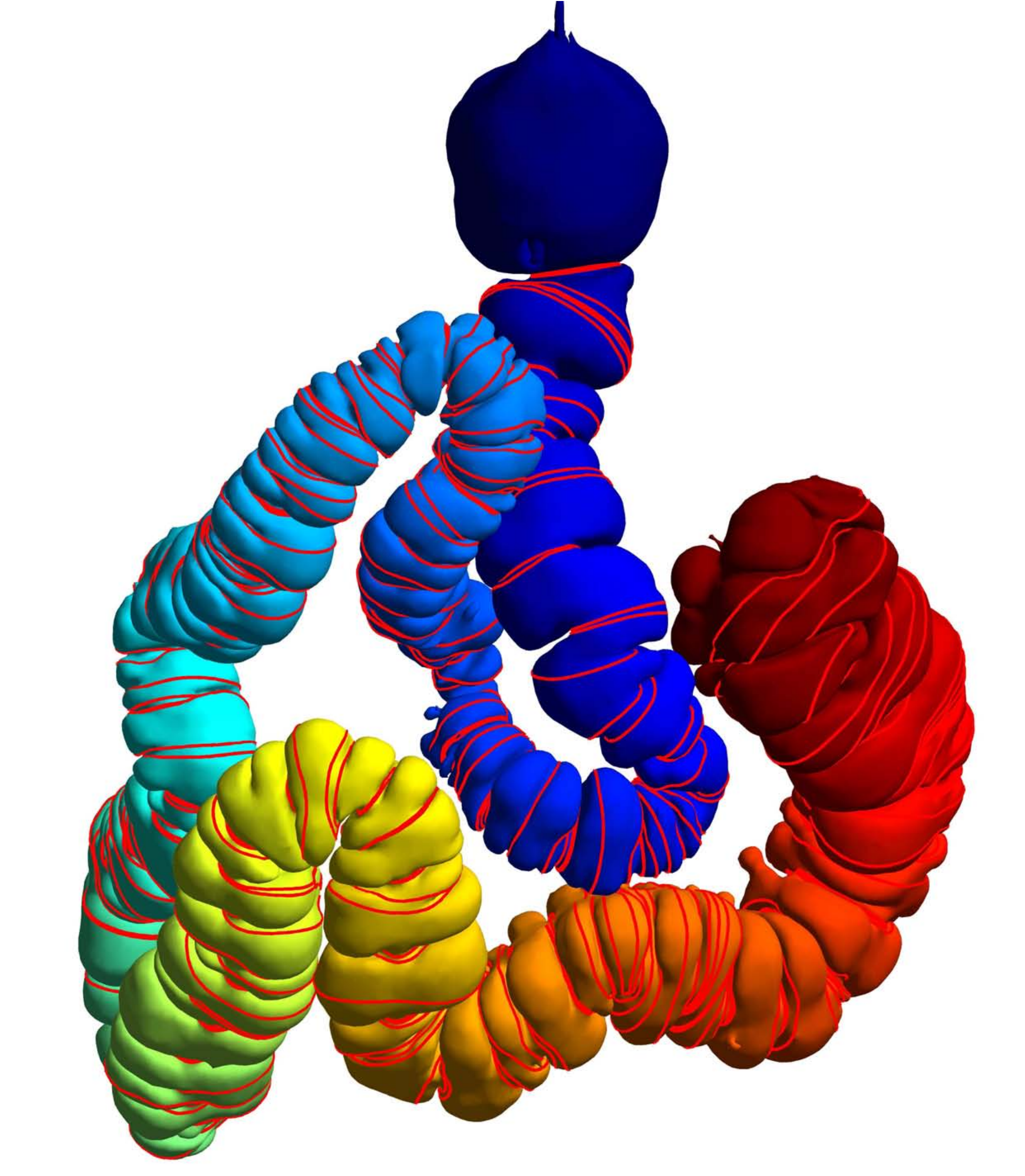} &
\includegraphics[width=0.16\textwidth,angle=180, origin=c]{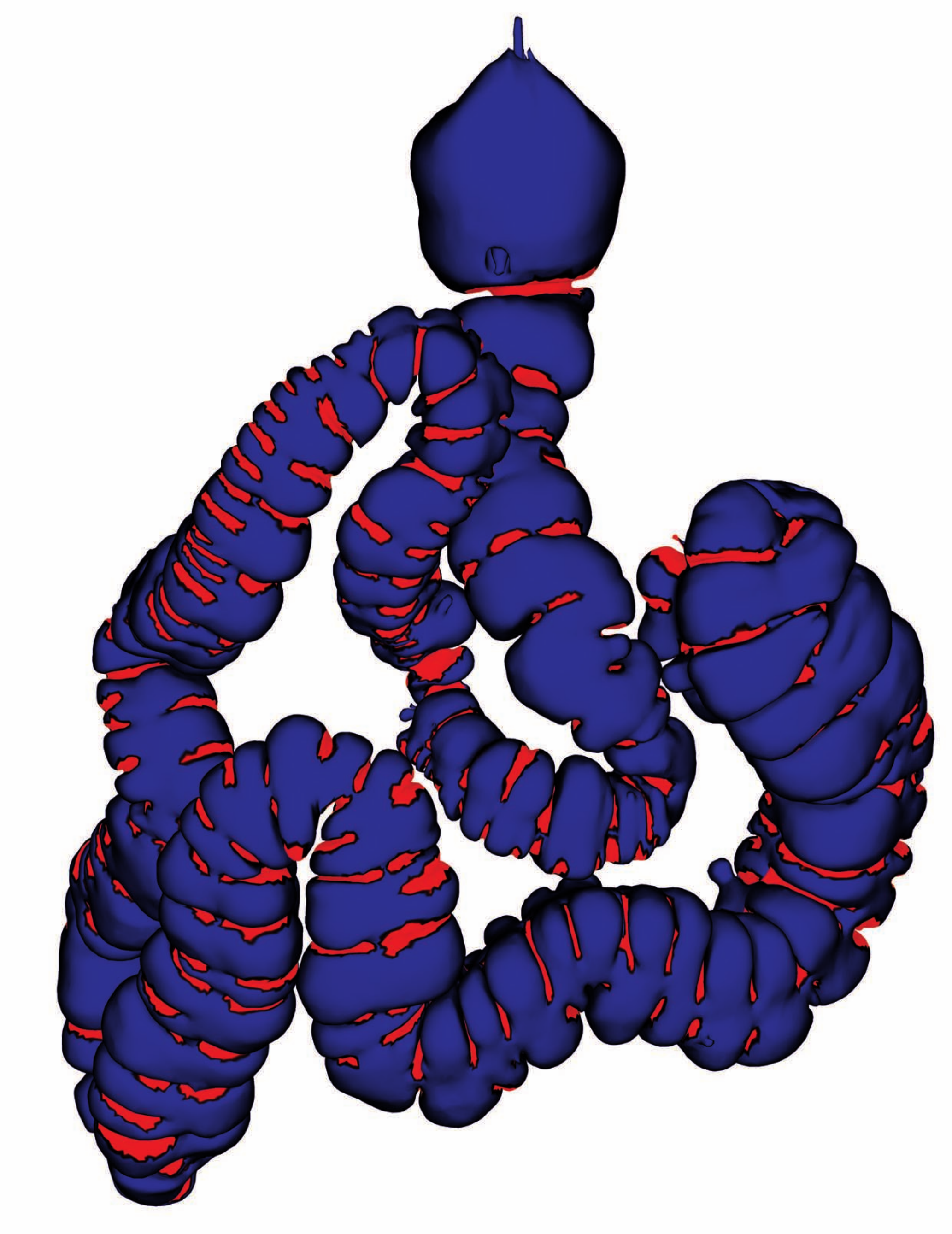} \\
(a) & (b) & (c) & (d) & (e)\\
\multicolumn{5}{c}{\includegraphics[width=0.95\textwidth]{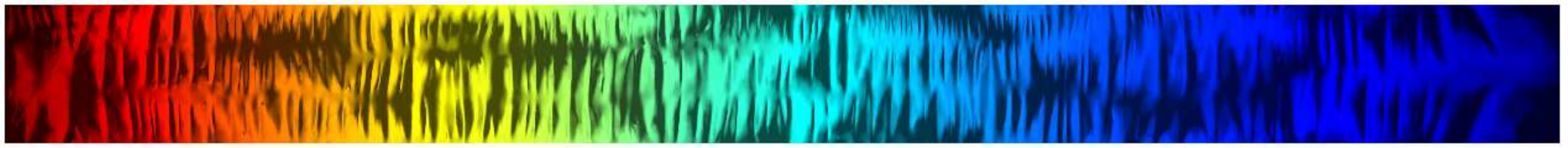}}\\
\multicolumn{5}{c}{(f) Flattened colon Fiedler vector representation}\\
\multicolumn{5}{c}{\includegraphics[width=0.952\textwidth]{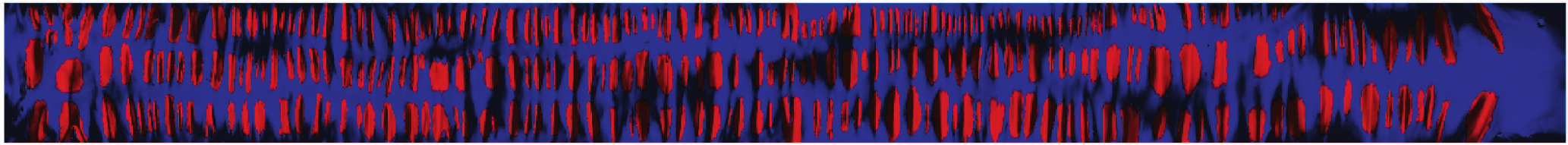}}\\
\multicolumn{5}{c}{(g) Flattened fold segmentation}
\end{tabular}
\end{center}
\vspace{-6mm}
\caption{Haustral fold detection and segmentation for a complete colon dataset. (a) Fiedler vector computed for a colon dataset. (b) Equi-sampled level sets computed based on (a). (c) Planes fitted to the level sets in (b) with the level sets projected to the planes. (d) Haustral fold contours detected based on the normalization of the planes in 2D. (e) Haustral fold segmentation based on the cross curves shown in Figure \ref{fig:fold_planes}. (f) Flattened Fiedler vector representation corresponding to the 3D view in (a). (g) Flattened haustral fold segmentation corresponding to the 3D view in (e).
\label{fig:colon_160}}
\vspace{-3mm}
\end{figure*}

In this section, we explain the computational algorithms for our registration method. We globally register the supine and prone colons based on their respective Fiedler vector representations. This is followed by bounded fold contour detection using 3D level set bundles computed using the Fiedler vector representation, and the eventual segmentation into individual folds based on the 2D normalization of the projected detected fold contour planes. These segmented folds are used as references on the globally registered colons to locally refine the registration and result in the final output of our registration algorithm.

\subsection{Fiedler Vector}

Suppose the colon surface is $(S,\mathbf{g})$, where $\mathbf{g}$ is the Riemannian metric. The colon surface is a topological cylinder  with two boundaries, $\partial S = \gamma_1 - \gamma_0$. We want to compute the Fiedler vector with the Dirichlet boundary condition
\[
    \left\{
    \begin{array}{lcl}
    \Delta_\mathbf{g} \xi &=& \lambda \xi, \lambda > 0\\
    \xi(p) &=& 1, p\in \gamma_1\\
    \xi(p) &=& 0, p\in \gamma_0
    \end{array}
    \right.
\]
where $\Delta_{\mathbf{g}}$ is the Laplace-Beltrami operator \cite{meyer:2003} induced by $\mathbf{g}$.

In the discrete setting, the colon surface is extracted from CT images, and approximated by a triangular mesh, denoted as $M=(V,E,F)$, where $V$ is the set of vertices, $E$ the set of edges, and $F$ the set of faces. The \emph{cotangent edge weight} is defined as follows: suppose $[v_i,v_j]\in E$ is an edge, shared by two triangular faces $[v_i,v_j,v_k]$ and $[v_j,v_i,v_l]$, then the edge weight is $w_{ij} = \cot \theta_{ij}^k + \cot \theta_{ji}^l$,
where $\theta_{ij}^k$ is the corner angle at vertex $v_k$ in the face $[v_i,v_j,v_k]$. Suppose a function is defined on the vertex, $f:V\to\mathbb{R}$. The \emph{discrete Laplace-Beltrami operator} is defined as $\Delta f(v_i) = \sum_{v_j\sim v_i} w_{ij}(f(v_j)-f(v_i))$, where $v_j\sim v_i$ means the vertex $v_j$ is adjacent to $v_i$. Therefore, the Laplace-Beltrami operator has a matrix representation $\Delta=(\delta_{ij})$, where
\[
    \delta_{ij} = \left\{
    \begin{array}{lr}
    -w_{ij}& v_i\sim v_j\\
    \sum_k w_{ik}& i=j\\
    0 & v_i\not\sim v_j
    \end{array}
    \right.
\]

We compute the eigen decomposition of $\Delta$; the eigenvalues are $\{\lambda_0,\lambda_1,\cdots, \lambda_n\}$ and the eigenvectors are $\{\eta_1,\eta_2,\cdots, \eta_n\}$. The first eigenvalue of $\Delta$ is $\lambda_0=0$ and the first eigenvector is $(1,1,\cdots,1)^T$. The first positive eigenvalue is $\lambda_1$ and the Fiedler vector is the corresponding eigenvector $\eta_1$. We scale $\eta_1$ such that the maximum value of $\eta_1$ is $1$ and the minimum value is $0$. In the following discussion, we treat $\eta_1$ as a function defined on the vertex set. Figure \ref{fig:supine0}(a) illustrates the Fiedler vector of a colon surface with color encoding, and Figure \ref{fig:supine0}(d) shows the Fiedler vector on the flattened colon surface.

\subsection{Global Registration}

Two colon surfaces can be registered directly based on their Fielder vectors. Suppose the two colon surfaces are $S_1$ and $S_2$, then the corresponding Fiedler vectors are $\xi_1$ and $\xi_2$, respectively. Furthermore, the boundaries of the two surfaces are $\partial S_k = \gamma_1^k - \gamma_0^k, k=1,2$. We first match the boundary curves using their arc lengths. We choose base points $p_1\in \gamma_0^1$ and $p_2\in \gamma_0^2$, where the base points are the extrema which are computed using the maximum and minimum Fiedler vector values. To avoid Fiedler vector flips between the corresponding supine and prone data, we manually make these maximum or minimum Fiedler vector values consistent on the rectum or the cecum for both the supine and prone datasets.  The arc length is then used to parameterize the boundary curves and normalize the total length to $2\pi$. The arc length parameterizations are denoted as $\theta_1$ and $\theta_2$. The mapping $\theta_1 \mapsto \theta_2$ gives the mapping from $\gamma_0^1$ to $\gamma_0^2$.

We compute the gradient field on the surfaces, denoted as $\nabla \xi_1$ and $\nabla \xi_2$ respectively, then we trace the integration curves of the gradient fields. The curve $\tau_1(t)$ represents an integration curve on $S_1$, $\tau_1(0)\in \gamma_0^1$ and $\tau_1(1)\in \gamma_1^1$. Similarly, $\tau_2(t)$ is an integration curve on $S_2$ starting from $\gamma_0^2$. A point $p\in S_1$ is the intersection of a level set of $\xi_1^{-1}(t), t\in [0,1]$ and an integration curve $\tau_1$ starting from a point on $\gamma_0^1$ with the arc length parameter $\theta_1$. The whole colon surface $S_1$ is parameterized by $(\theta_1,t)$, and $S_2$ is globally parameterized by $(\theta_2,t)$. The initial global registration is given by the mapping $\varphi:(\theta_1,t)\mapsto (\theta_2,t)$.

\subsection{Fold Detection}

Suppose $S$ is a colon surface with Fiedler vector $\xi$. We denote the level sets of $\xi$ as $\gamma_t$, where $t\in [0,1]$, namely $\gamma_t = \xi^{-1}(t)$. We uniformly sample the level sets on the colon surface and obtain a family of level sets $\{\gamma_t\}$, as shown on the colon surface in Figure \ref{fig:supine0}(b) and the flattened colon surface in Figure \ref{fig:supine0}(e). We also compute the integration curves of the gradient field of $\xi$ and obtain a family of integration curves $\{\tau_\theta\}$, $\theta \in [0, 2\pi)$, as shown in Figure \ref{fig:fold_planes}, where the red loops are the level sets $\{\gamma_t\}$ and the yellow curves are the integration curves $\{\tau_\theta\}$.

We compute the normal curvature of points $\tau_\theta(t)$. If a point $(\theta,t)$ is in a fold area, then the normal curvature of $\tau_\theta(t)$ is negative; in other flat or convex areas, the normal curvature of $\tau_\theta(t)$ is zero or positive. We compute the \emph{inflection points} of the integration curves, where the normal curvatures are $0$'s, and the points with minimal normal curvatures. Fixing an integration curvature $\tau_\theta$, suppose $t_0<t_1<t_2$, where $t_0$ and $t_2$ are inflection points, $t_1$ is the minimal curvature point, then the level sets $\gamma_t$, $t\in [t_0,t_2]$ form a level set bundle. In this way, we compute the clusters of level sets. We discard the initial uniformly sampled level sets, which do not belong to any bundles. The level set bundles are shown in Figure \ref{fig:colon_160}(d).

We densely sample the level sets within each bundle, and compute the normal curvatures. The points $(\theta,t)$ with negative normal curvature of curve $\tau_\theta$ are in the fold area. In this way, we can detect the fold contour, as shown in Figure \ref{fig:colon_160}(e).

%Segmentation artifacts can result in collapsed regions. In the case of these regions, we do not extract folds since they do not show any meaningful information. We set a threshold on the level set bundle such that if the encompassing circle is below 0.1, when normalized in 2D, we mark that region as collapsed, as shown in Figure \ref{fig:colon_102}.

Bowel preparation, done prior to VC, might lead to local under-distention of some colon regions, which we refer to as collapsed regions. In this case, we ignore the extraction of folds in these regions since these folds do not exhibit any meaningful information. We set a threshold on the level set bundle such that if the encompassing circle is below 0.1 in radius, when normalized in 2D, we mark that region as locally under-distended or collapsed, as shown in Figure \ref{fig:colon_102}. However, severe cases of local under-distension can lead to a complete collapse of the colon at a particular region, resulting in multiple colon segments during the segmentation process. Our method does not cater to these severe cases, which will be a focus of future research.

\subsection{Haustral Fold Segmentation}

\begin{figure}[t]
\begin{center}
\includegraphics[width=0.33\textwidth, height=0.21\textheight]{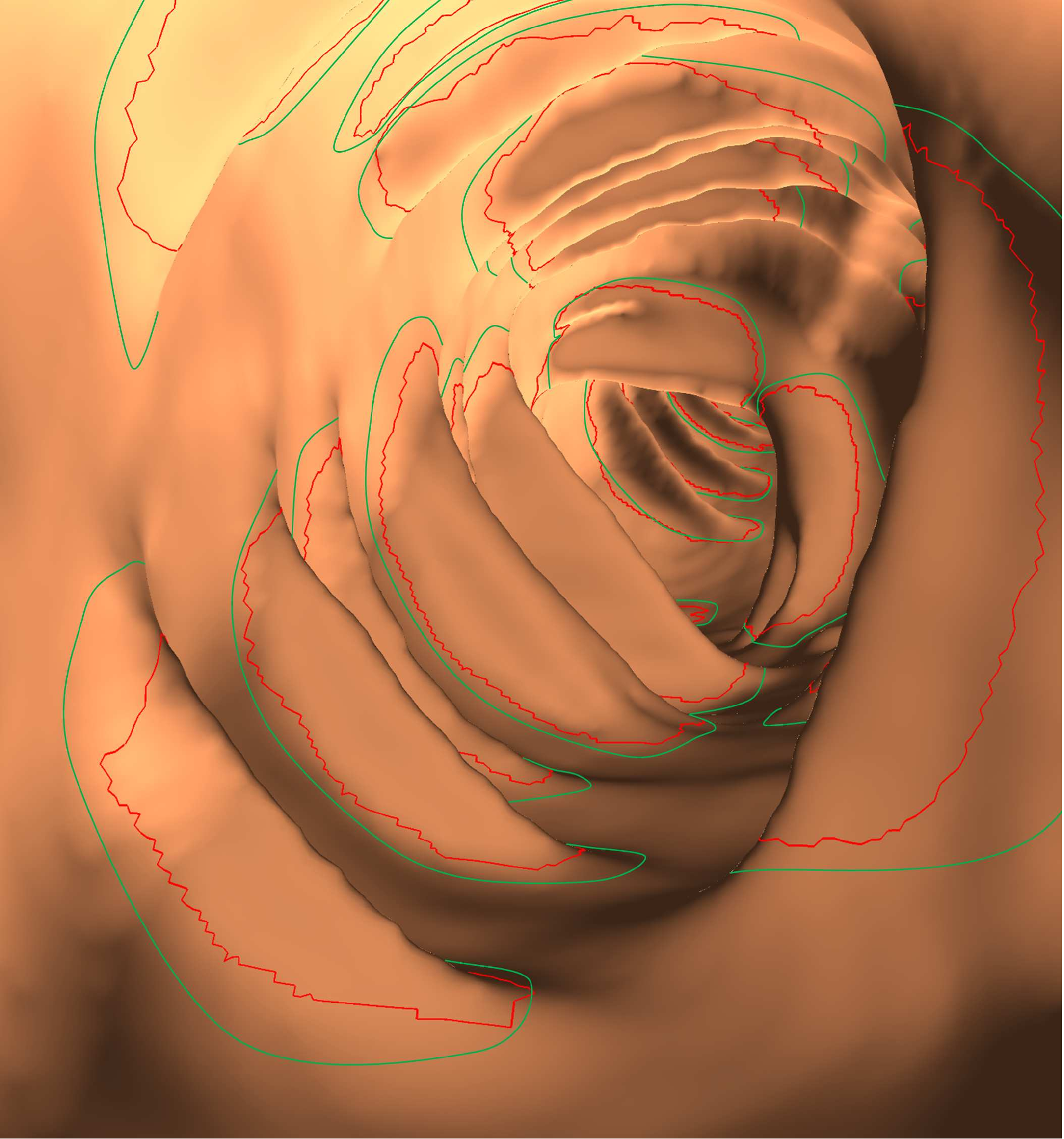}\\
\end{center}
\vspace{-4.5mm}
\caption{Endoluminal view of the haustral fold segmentation with red contours denoting the automatic segmentation and green contours denoting the manual segmentation.
\label{fig:endoluminal_folds}}
%\vspace{-1mm}
\end{figure}

For each level set $\gamma_t$, we find a best fit plane $\pi_t$ as follows. We choose samples on $\gamma_t$, denoted as $\{p_1,p_2,\cdots,p_n\}$, and the center of the samples is given by $c = 1/n \sum_{i=1}^n p_i$. We compute the covariance matrix of the samples, $A = 1/n \sum_{i=1}^n (p_i - c) \bigotimes (p_i-c)$, then compute the eigen decomposition of $A$. The eigenvectors are $\{\eta_1, \eta_2, \eta_3\}$, which form an orthonormal, then the fitting plane $\pi_t$ goes through the center $c$ and is spanned by $\eta_1$ and $\eta_2$, namely, $\pi_t$ is given by the equation $\langle p-c, \eta_3\rangle = 0$. The fitting planes of the level sets are illustrated in Figure \ref{fig:colon_160}(c).

We then project the level set $\gamma_t$ onto the fitting plane $\pi_t$. Suppose $p\in \gamma_t$, then the projection on the plane is given by $p - \langle p, \eta_3\rangle \eta_3$. We denote the projected level set as $\tilde{\gamma}_t$. Suppose $\pi_{t}$ and $\pi_s$ are two fitting planes, with centers $c_t$ and $c_s$, we align them together as follows: we shift $c_t$ to $c_s$ by a translation, then rotate $\pi_t$ to coincide with $\pi_s$ by a rotation $\mathcal{R}$, such that the rotation angle is the angle between the normals to the two planes and the rotation axis is along the cross product of the two normals. In this way, we align the fitting planes of all level sets within a bundle, and the spatial level set curves become a cluster of planar curves as shown in Figure \ref{fig:fold_planes}(c). We denote each projected and aligned level set as $\hat{\gamma}_t$.

An integration curve $\tau_\theta(t)$ on the original surface becomes a planar curve $\hat{\tau}_\theta(t)$ by mapping $\gamma_t(\theta)$ to $\hat{\gamma}_t(\theta)$. Suppose the level set bundle is $\{\gamma_t\}$, where $t_0\le t \le t_1$. We compute the total length of each $\hat{\tau}_\theta$ and find the local minima with respect to $\theta$. Then each local maximum corresponds to a haustral fold. This procedure produces the haustral fold segmentation, as shown in Figure \ref{fig:colon_160}(e) for the 3D view and Figure \ref{fig:colon_160}(g) for the flattened view. We compare our automatic haustral fold segmentation with manual segmentation as shown in Figures \ref{fig:endoluminal_folds} and \ref{fig:aut_man}, where the green contour shows the manual segmentation results and the red contours show the automatic results. One can see that the automatic segmentation has very high accuracy.

\subsection{Local Registration Refinement}

The folding areas and the segmented haustral folds are used as anatomical references. From the initial registration, we can find the correspondences among the haustral folds on colon surfaces obtained from the supine/prone/side data. We then locally deform the initial mapping in order to align the haustral folds more accurately. We denote the two flattened colon surfaces as $S_1$ and $S_2$ respectively, $\varphi: S_1\to S_2$ is the mapping between them.  We define the characteristic function $\chi_k: S_k\to\mathbb{R}$, $k=1,2$, where $\chi_k(p)$ equals $1$ if $p$ is inside a haustral fold, and $0$ otherwise. Then we smooth the characteristic functions out by convolving a Gaussian filter. We define an energy for the mapping $\varphi$ as follows:
\[
    E(\varphi) = \int_{S_1} | \chi_1(p) - \chi_2\circ\varphi(p) |^2 dA + \int_{S_1} |\nabla \varphi|^2 dA.
\]
The first term ensures that haustral folds match haustral folds, and the second term measures the smoothness of the mapping.  We obtain the Euler-Lagrange equation, and the following flow minimizes the energy,
\[
    \frac{\partial \varphi(p,t)}{\partial t} = -(\chi_1(p)-\chi_2\circ\varphi(p)) \nabla \chi_2(\varphi(p)) + \Delta \varphi(p).
\]
By deforming the mapping along the flow, the registration is improved significantly. Figure \ref{fig:registration_compare} shows the registration result using the haustral folds as references.

\begin{figure}[t!]
\begin{center}
\begin{tabular}{c}
\includegraphics[width=0.40\textwidth,height=0.13\textheight]{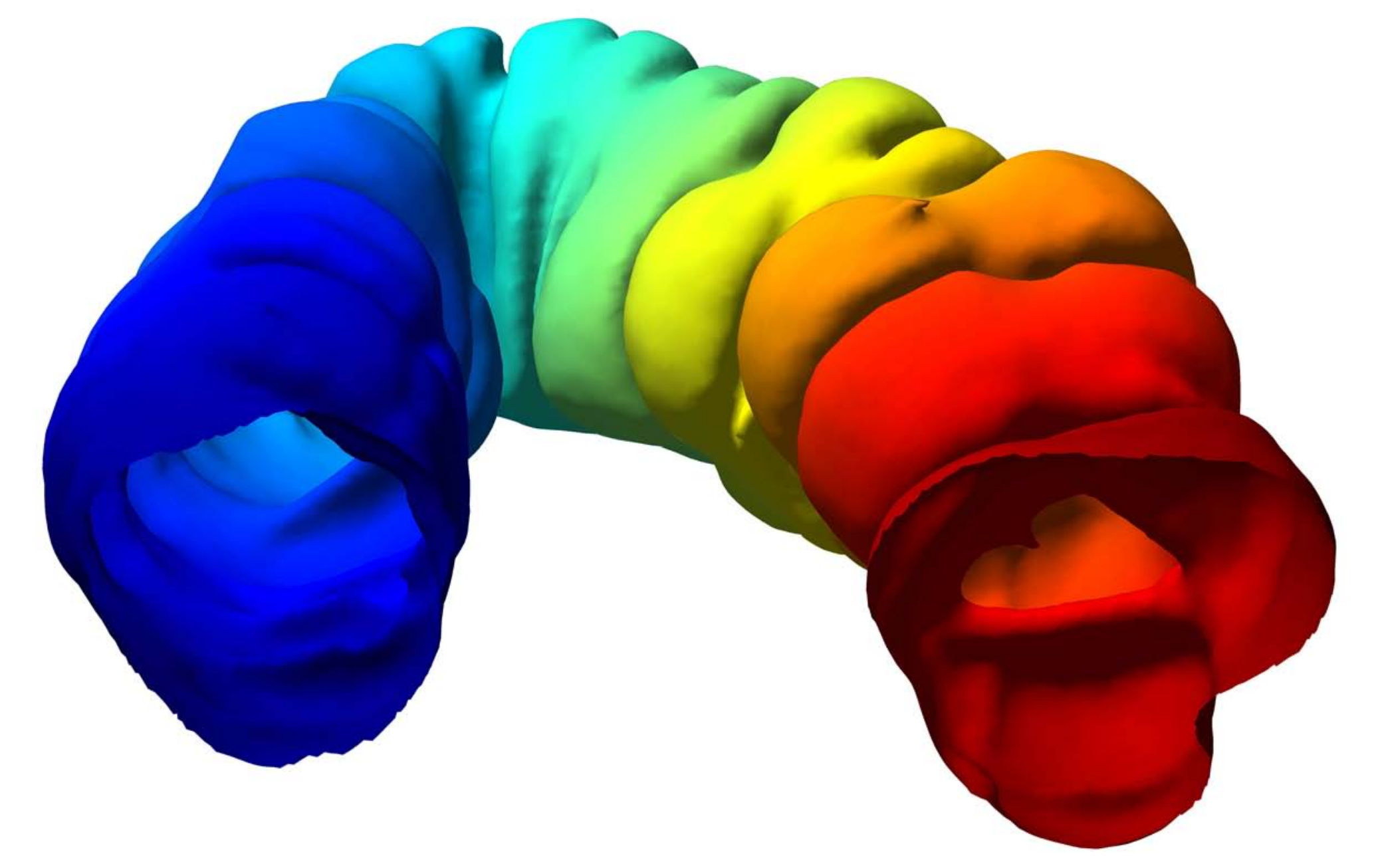}\vspace{-1mm}\\
(a) Fiedler vector on colon segment\vspace{1.5mm}\\
\includegraphics[width=0.46\textwidth]{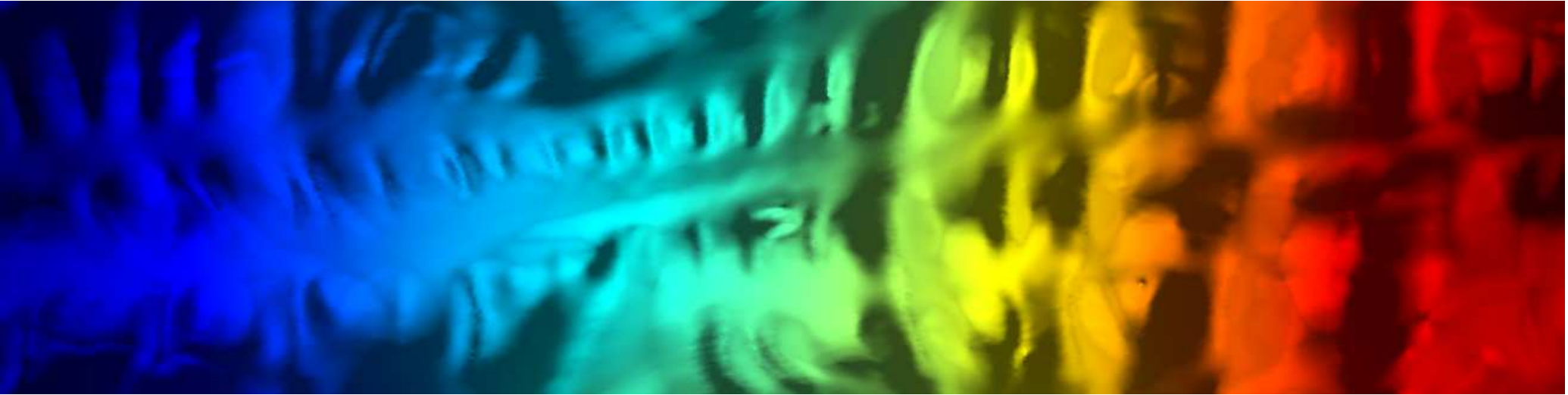}\\
(b) Flattened colon segment with a cut via geodesic path\vspace{1.5mm}\\
\includegraphics[width=0.46\textwidth]{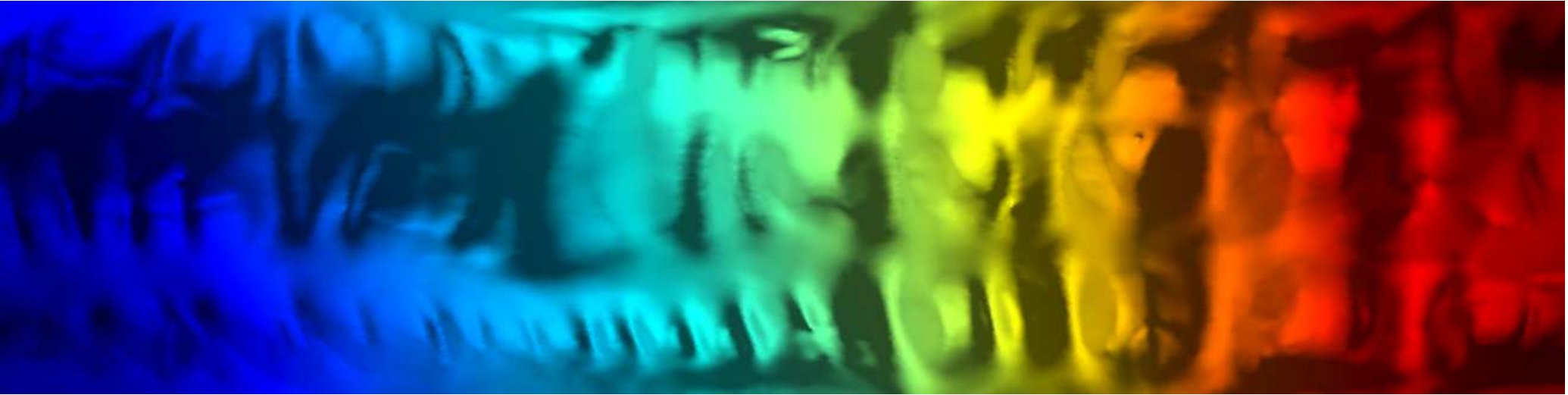}\\
(c) Flattened colon segment with a cut via our fold segmentation\\
\end{tabular}
\end{center}
\vspace{-5mm}
\caption{Consistent cuts. (a) Fiedler vector representation on a colon segment. (b) For two given points on the boundaries, a geodesic path cuts through the folds and chops the folds on the boundaries while flattening. (c) Due to our accurate segmentation, we can trace consistent cuts automatically and keep the haustral folds intact during flattening.
\label{fig:automatic_cut}}
%\vspace{-1mm}
\end{figure}

%\begin{figure*}[t!]
%\begin{center}
%\begin{tabular}{c}
%\includegraphics[width=0.95\textwidth]{figures/folds_seg/colon_160_S_simp_flatten_folds_other_alg.pdf}\\
%(a) Fold detection algorithm using heat diffusion and fuzzy C-means clustering \cite{chowdhury:2010:colonic}\\
%\includegraphics[width=0.95\textwidth]{figures/folds_seg/colon_160_S_simp_flatten_folds.pdf}\\
%(b) Our fold detection and segmentation algorithm
%\end{tabular}
%\end{center}
%\vspace{-6mm}
%\caption{Fold detection evaluation on a flattened colon segment. (a) Fold detection algorithm \cite{chowdhury:2010:colonic} and (b) our fold detection algorithm.
%\label{fig:fold_detection_eval}}
%\end{figure*}

\section{Visualization}
\label{sec:vis}
%We can obtain more effective flattening visualizations using more consistent cuts delineated through segmented folds. Moreover, the Fiedler vector representation provides easier visualization to find correspondences between 3D and flattened visualizations while allowing for better polyp localization via fold segmentation/labeling.

Due to the use of the Fiedler vector representation throughout the entire colon, it becomes possible to easily co-locate positions between two scans, either manually or automatically.  Using consistent cuts along the colon based on the segmented folds, we achieve more effective flattened visualizations than other works.  The Fiedler vector representation also provides easier visualizations to find correspondences between the 3D model and the flattened view while allowing for better polyp localization via labeling of segmented folds.

\begin{figure}[t!]
\vspace{-1mm}
\begin{center}
\begin{tabular}{cc}
\includegraphics[width=0.165\textwidth]{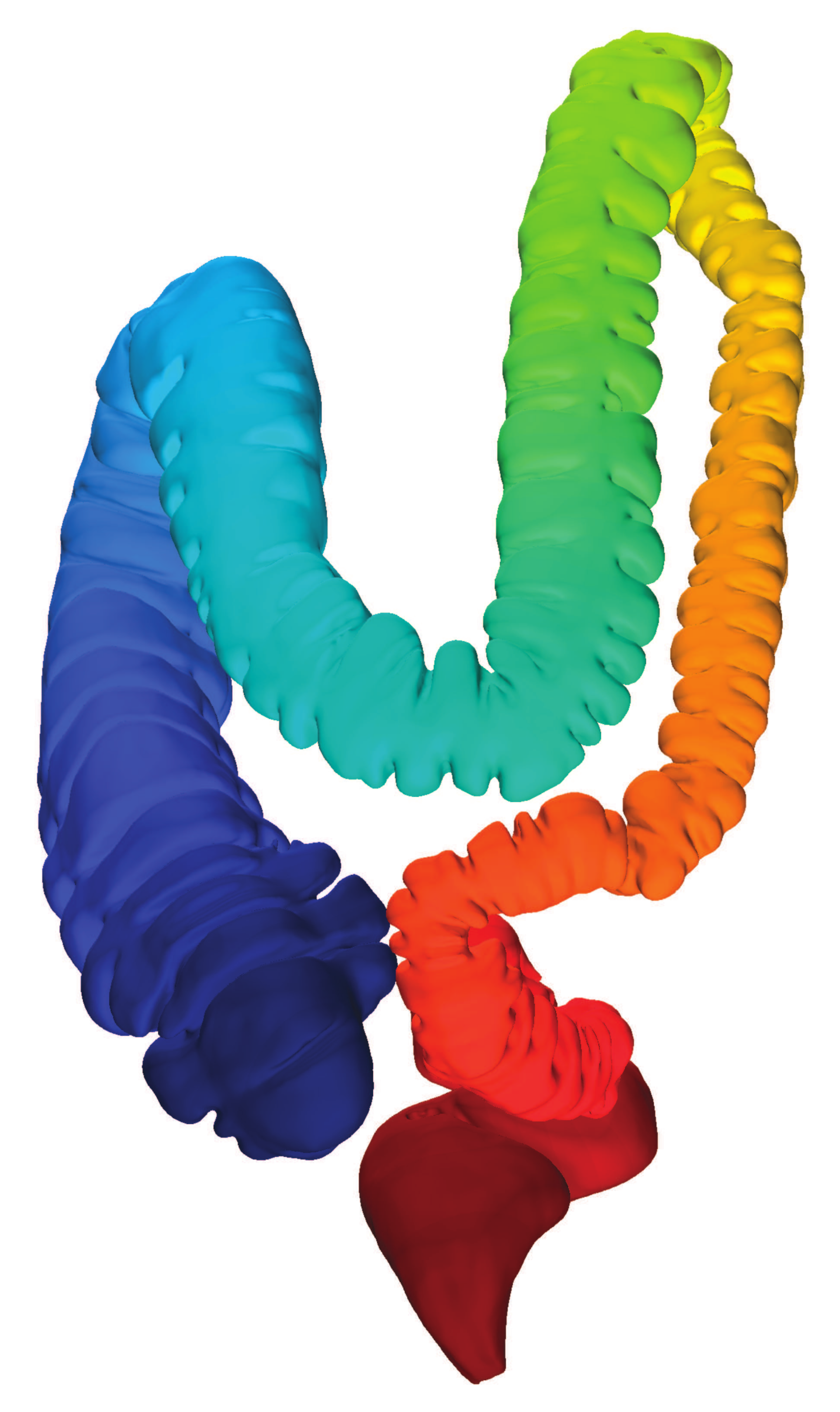}&
\includegraphics[width=0.165\textwidth]{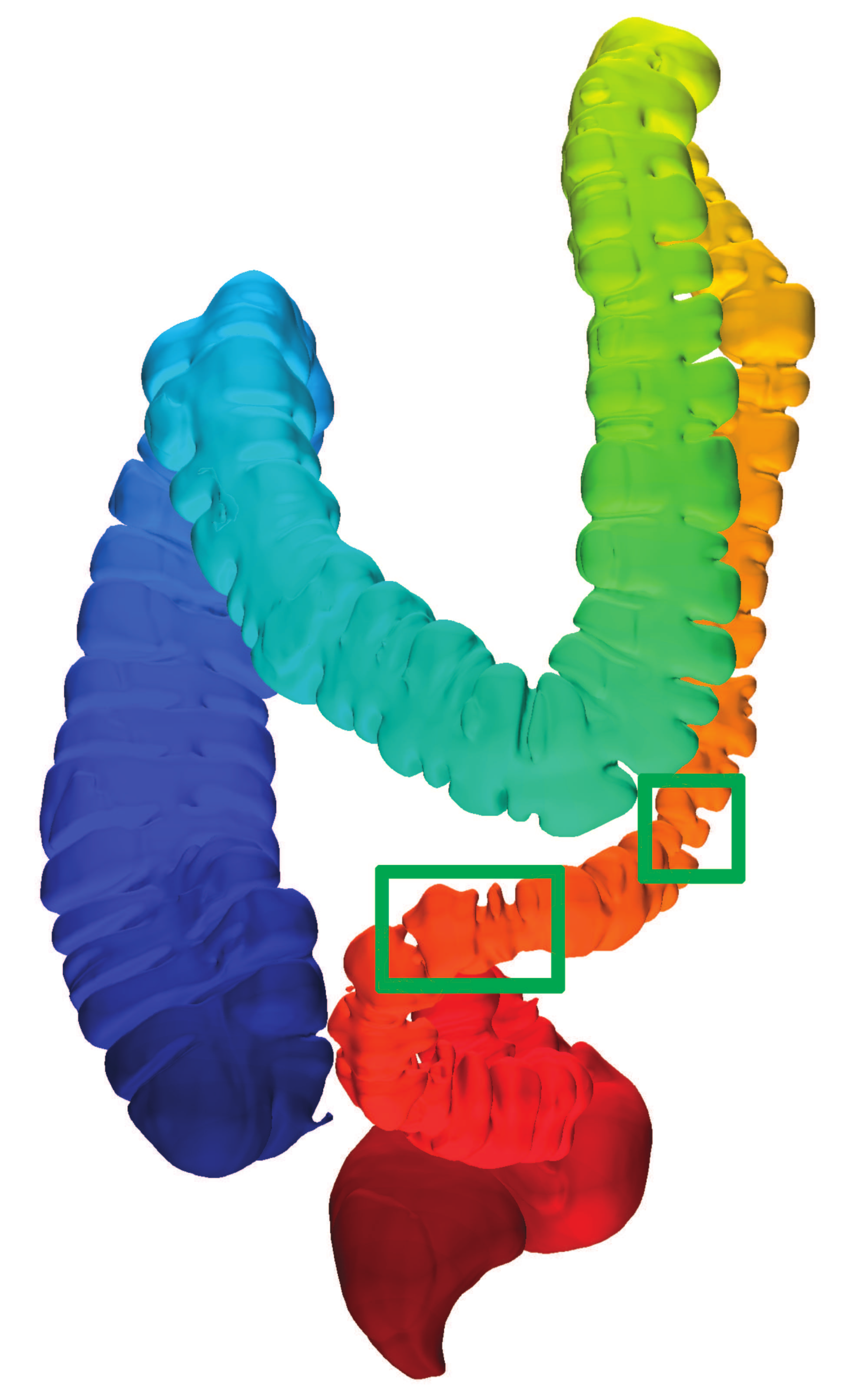}\\
(a) Supine & (b) Prone\\
\includegraphics[width=0.155\textwidth]{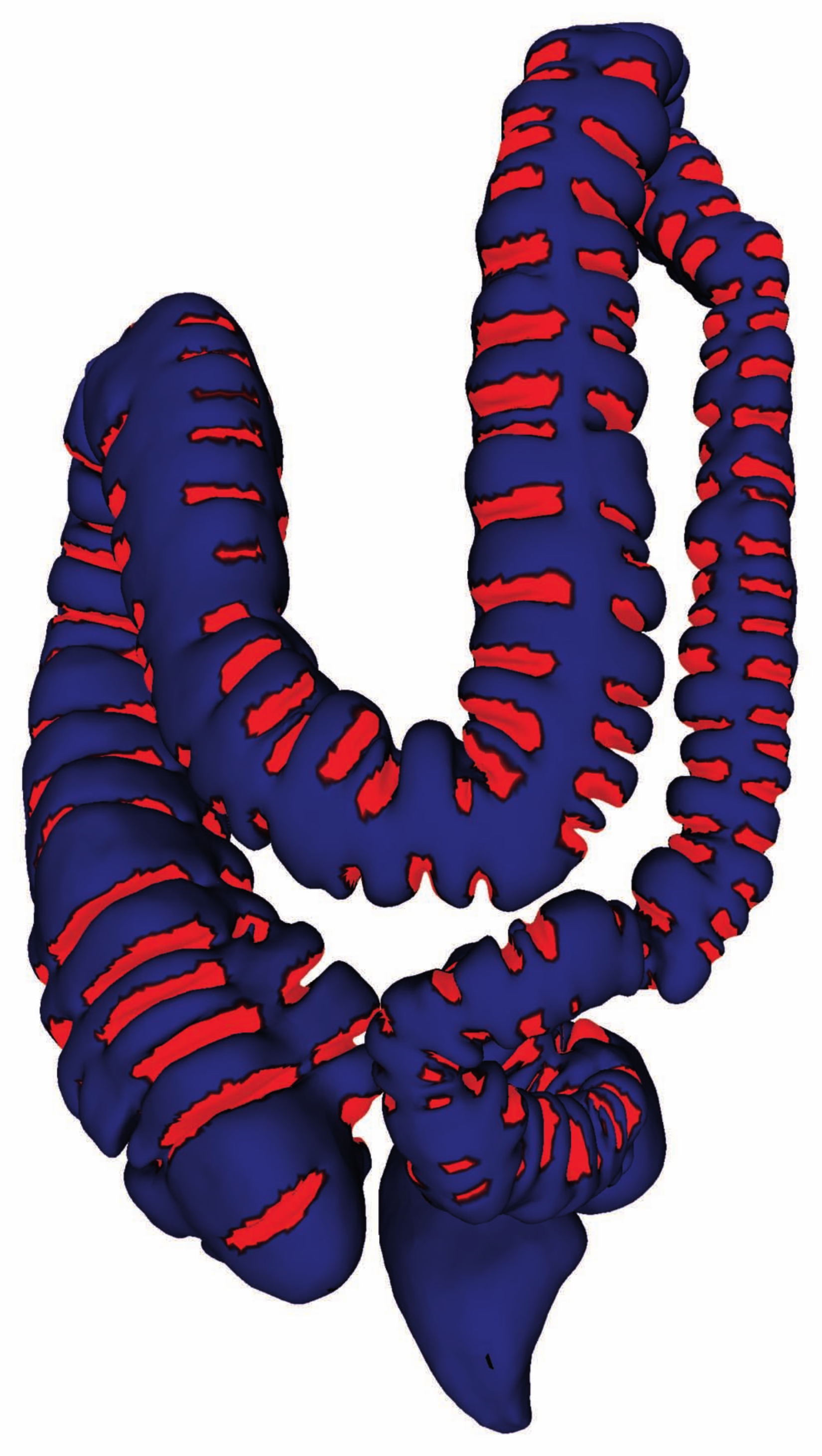}&
\includegraphics[width=0.155\textwidth]{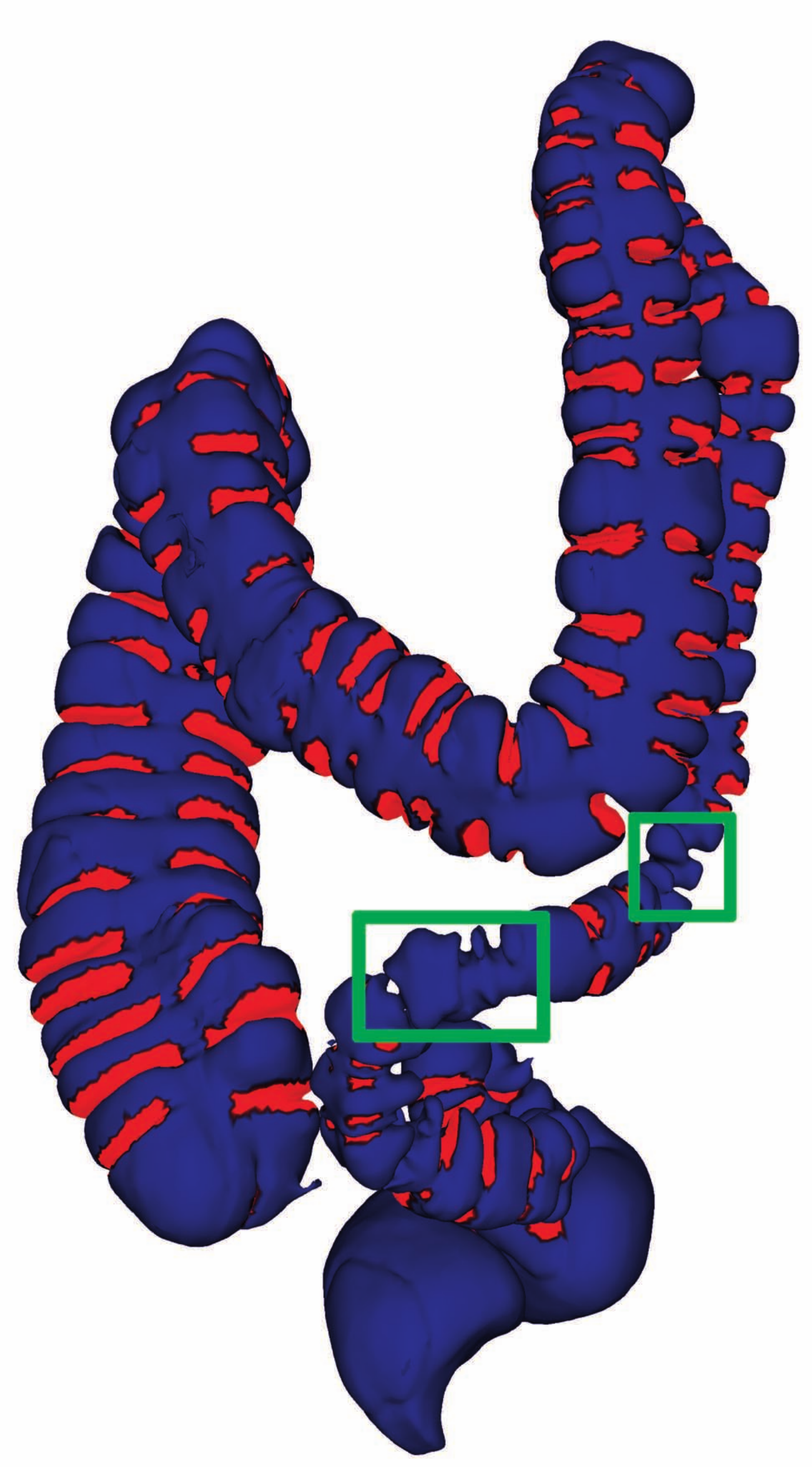}\\
(c) Supine & (d) Prone\\
\end{tabular}
\end{center}
\vspace{-5mm}
\caption{Fiedler vector computation for (a) supine and (b) prone colons and the corresponding (c) and (d) fold segmentation results. The green boxes in (b) and (d) show the collapsed regions in the prone dataset.
\label{fig:colon_102}}
%\vspace{-4mm}
\end{figure}

\subsection{Corresponding Supine and Prone Visualization}
The Fiedler vector representation can be used to automatically visualize corresponding regions in multiple patient orientations and in respective 3D and 2D mapped views. The consistent endoluminal views in Figures \ref{fig:teaser} and \ref{fig:supineside} highlight the polyps in supine-prone and supine-side orientations, respectively. In order to visualize both supine and prone in a consistent manner, we map the camera frustum viewing the supine mucosa (inner surface) to the camera frustum viewing the prone mucosa. This is done by computing the means of the level sets based on the Fiedler vector which result in corresponding centerlines on both supine and prone. We then match the camera orthonormal coordinate frames for the two views.

The centerlines are matched based on the eigenfunction bundles following the registration algorithm, outlined in Section \ref{sec:algorithm}. Given a point $c_0$ on the discretized supine centerline, we want to find the corresponding point $c_1$ on the discretized prone centerline. We first find the point $d_0$ on the supine centerline, which is the point nearest $c_0$ in $\mathbb{R}^3$. Given this point $d_0$, the corresponding point $d_1$ on the prone centerline is found at the same index.  $c_1$ can then be found as the point on the prone centerline closest in $\mathbb{R}^3$ to $d_1$.

We focus our endoluminal correlation work on situations where the user is looking at something on the colon wall and wants to view the corresponding location in the other scan.  Using the correspondence along the skeletons, it is possible to create a correlated automatic navigation, though we found running two fly-throughs side-by-side to be distracting and more of a hindrance than help. Orienting rotation around the centerline for two automatic navigation views is possible based on the individual haustral fold registration on the corresponding supine and prone colons.

\subsection{Flattened Visualizations}
Since the Fiedler vector representation captures extreme points on a given colon dataset, we can use these extreme points along with the fold segmentation results to trace a consistent cut throughout the colon. This cut can then be used as input to a quasi-conformal mapping algorithm \cite{zeng2010supine} to flatten the colon to a 2D plane. In essence, the following steps are used to automatically extract a consistent cut: (1) extract the extreme points on the colon dataset, (2) remove the 1-ring vertex set corresponding to the extreme vertices, (3) pick two haustral folds (in the same level set bundle) closest to one of the extreme points, (4) connect the mean of the endpoints of these folds to the next level set bundle of haustral fold endpoints, closest to the centroids of the folds in the previous bundle, and (5) repeat the previous step until the other extreme point is reached. The result of this process, compared to a typical cut based on a geodesic path, is illustrated in Figure \ref{fig:automatic_cut}.

To the best of our knowledge, there are no automatic algorithms for extracting a consistent cut throughout the colon. Normally, two extreme points are manually picked on a given colon dataset and a geodesic path is computed from one point to the other and this constitutes the cut. However, the problem with this cut is that it crosses the folds when a sharp bend is incurred, as shown in Figure \ref{fig:automatic_cut}, whereas our method takes into account individual haustral folds before tracing the cut which allows for consistency.

In addition to greater consistency, our cut produces less angle distortion as compared to the geodesic path cut. We quantify the angle distortion metric of the resultant flattened colon by using the signed singular values of the Jacobian of the transformation for each triangle \cite{liu:2008local}. Small angular distortion is indicated by a distortion value approaching 2. The colon flattened via geodesic path cut in Figure \ref{fig:automatic_cut}(b) produces an overall angle distortion value of 2.181, whereas our method in Figure \ref{fig:automatic_cut}(c) produces a value of 2.126. 
\section{Experimental Results}
\label{sec:experiments}

We have validated our algorithms using real VC colon data from the publicly available National Institute of Biomedical Imaging and Bioengineering (NIBIB) Image and Clinical Data Repository provided by the National Institute of Health (NIH). We perform electronic colon cleansing incorporating the partial volume effect \cite{wang:2006:tbme}, segmentation with topological simplification \cite{hong:2006:tvcg}, and reconstruction of the colon surface as a triangular mesh via surface nets \cite{Gibson98SurfaceNets} on the original CT images in a pre-processing step. Though the size and resolution of each CT volume varies between clinical datasets, the general data size is approximately $512 \times 512 \times 450$ voxels with a resolution of approximately $0.7 \times 0.7 \times 1.0$ mm.
%In this paper, the colon surface is modeled as a topological cylinder and discretely represented by a triangular mesh.

We have developed our algorithms using generic C++ on the Windows 7 platform. The linear systems for the Laplace equation were solved using the Matlab C++ library. All of the experiments are conducted on a workstation with a Core 2 Quad 2.50GHz CPU with 4GB RAM. On average, the Fiedler vector computation takes 3.5 seconds, the haustral fold detection and segmentation takes 3.2 seconds, and the final local refinement takes 1.8 seconds.

\begin{figure*}[ht!]
\begin{center}
\begin{tabular}{c}
\vspace{-1.5mm}\includegraphics[width=0.95\textwidth]{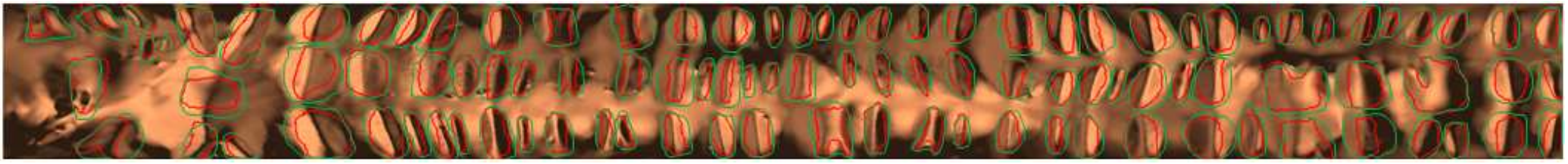}\\
\vspace{-1.5mm}\includegraphics[width=0.95\textwidth]{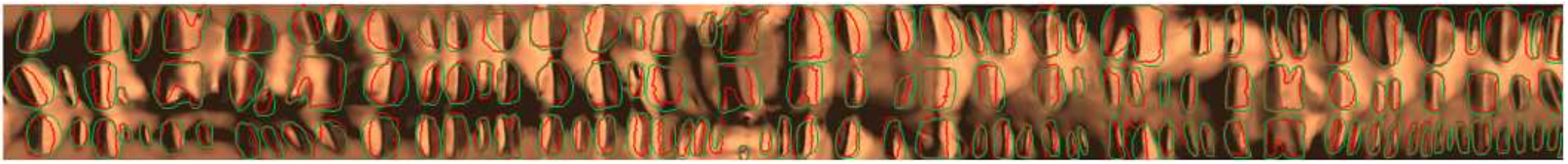}\\
\vspace{-7mm}\includegraphics[width=0.95\textwidth]{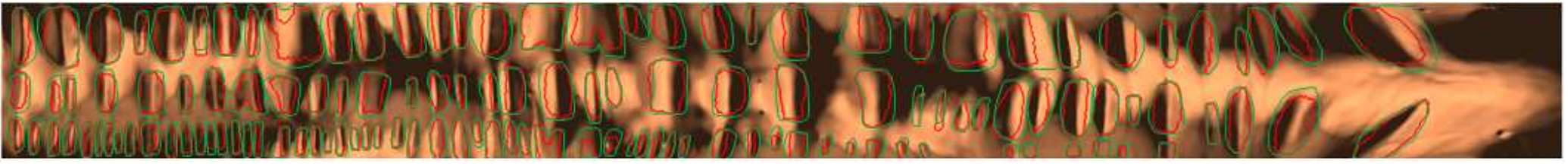}\\
\end{tabular}
\end{center}
\caption{Automatic fold segmentation (red contours) and manual fold segmentation (green contours) of a complete colon.
\label{fig:aut_man}}
\vspace{-4mm}
\end{figure*}

\begin{table}[!t]
\setlength{\tabcolsep}{5pt}
\centering
%\scriptsize
%\captionsetup{labelformat=empty}
\caption{Quantitative results from the 20 patient scans. Columns 2--6 indicate for each scan the number of ground truth folds, sensitivity (\# true positives divided by \# ground truth positives), false negatives, false positives, and the segmented-area ratio. The last row denotes the averages by scan and standard deviations of columns 2-6.}
\label{tab:folds}
\vspace{-3mm}
\begin{tabular*}{0.48\textwidth}{lccccc}
\hline
\hline
Scan ID & \# True & Sensitivity & \#FNs & \#FPs & $SAR$ \\
%        &   truth  &             &       &      &     \\
\hline
\hline
    1-Supine   &   285    &    0.94    &  18   &  0   & 0.82 \\
    1-Side   &   272    &    0.94    &  16   &  1  & 0.92 \\
    2-Supine   &   191    &    0.86    &  27   &  0   & 0.87 \\
    2-Side   &   169   &    0.83   &  28   &  0   & 0.86 \\
    3-Supine   &   293    &    0.99    &  2   &  0   & 0.93 \\
    3-Prone   &   268    &    0.97   &  7   &  1   & 0.81 \\
    4-Supine   &   274    &    0.93    &  18   &  0   & 0.91 \\
    4-Prone   &   251    &    0.90    &   24   &  0   & 0.83 \\
    5-Supine   &   165    &    0.84    &  26   &  1   & 0.86 \\
    5-Prone  &   152    &    0.82   &  27   &  0   & 0.76 \\
    6-Supine  &   223    &    0.94    &   13   &  0   &  0.86 \\
    6-Prone  &   202    &    0.88   &  25  &  1   & 0.76 \\
    7-Supine  &   196    &    0.88    &  24   &  1   & 0.91 \\
    7-Prone  &   161    &    0.94   &  9   &  0   & 0.78 \\
    8-Supine  &   247    &    0.97    &  8   &  0   & 0.94 \\
    8-Prone  &   210    &    0.91   &   19  &  0   & 0.86 \\
    9-Supine  &   187    &    0.82    &  33   &  1   & 0.84 \\
    9-Prone  &   158    &    0.91   &  15  &  0  & 0.75 \\
    10-Supine  &   173    &    0.87    &  22   &  0   & 0.82 \\
    10-Prone  &   142    &    0.92   &  11   &  0   & 0.89 \\
\hline
\hline
Average & 211.0 $\pm$ & 0.90 $\pm$ & 18.6 $\pm$ & 0.30 $\pm$  & 0.85 $\pm$\\
        &   49.4    & 0.05     & 8.36    & 0.47    &   0.06  \\

\hline
\hline
\end{tabular*}
%\vspace{-3mm}
\end{table}

\subsection{Fold Detection and Segmentation Evaluation}
\label{subsec:folds_eval}

To evaluate our fold segmentation algorithm, we manually segmented the folds on 10 supine-prone and supine-side colon pair datasets. These manual segmentations were approved by a VC-trained radiologist who reviewed them. The accuracy of the proposed segmentation algorithm was measured on a per-fold basis, using the $SAR$ metric \cite{zhu2013haustral}. Given the manually-established ground truth, a fold was assumed to be successfully detected if more than 50\% of its area has been detected. Hence, we can calculate the true positives (TP), false positives (FP), true negatives (TN), false negatives (FN), sensitivity ($=TP/(TP+FN)$), and specificity ($=TN/(FP+TN)$). We calculate the $SAR$ as follows:\vspace{-2mm}
\[
    SAR = A_0 / (A_t + A_d - A_0)
\vspace{-2mm}\]
where $A_t$ is the ground truth area of the fold. $A_d$ is the area of the segmented fold, and $A_0$ is the area of the overlap of the above two areas. $SAR$ is defined as the ratio between the area of the intersection and the area of the union of the expert-drawn folds and the automatically segmented folds. In essence, a larger $SAR$ suggests a better segmentation result, and $SAR = 1$ indicates that the segmented fold perfectly matches the ground truth without any over- or under-segmentation.

Table \ref{tab:folds} lists the quantitative results from the 20 patient scans. About 90.4\% of all the ground truth folds were detected with approximately 18 missed per dataset, under the assumption that a fold would be detected if more than 50\% of its area has been segmented. The missed folds were mostly shallow and flat shaped structures. The FP detections were very few for each scan. For all the detected folds, the $SAR$ is approximately 84.8\%, indicating that the automatically segmented fold boundaries have matched fairly well with those of the manually-drawn counterparts. A visual comparison of the detection accuracy of our algorithm with a previous fold detection method using heat diffusion and fuzzy C-means clustering is provided as part of the Appendix A.2.

\begin{figure*}[ht!]
\begin{center}
\begin{tabular}{c}
\vspace{-0.5mm}\includegraphics[width=0.945\textwidth]{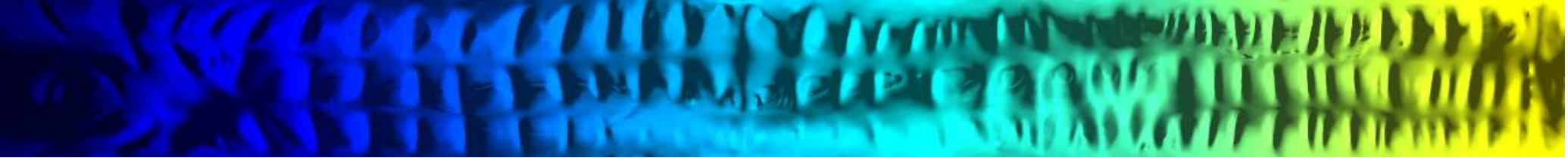}\\
\includegraphics[width=0.945\textwidth]{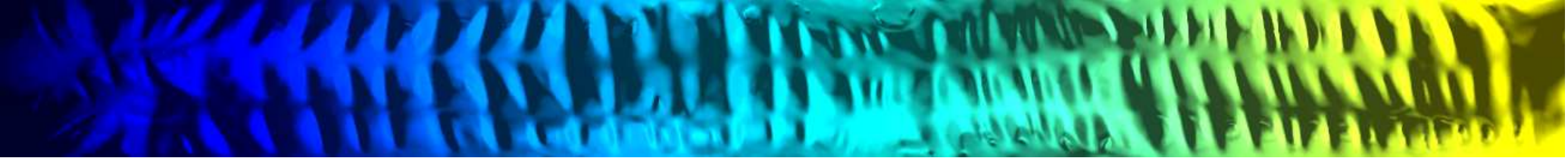}\vspace{-0.5mm}\\
(a) Unregistered supine and prone colon section 1\vspace{0.5mm}\\
\vspace{-0.5mm}\includegraphics[width=0.945\textwidth]{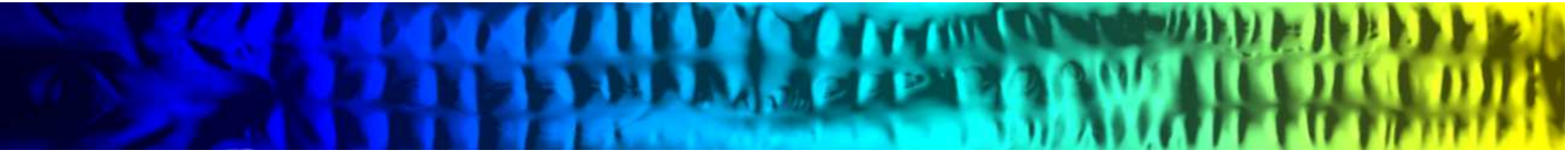}\\
\includegraphics[width=0.945\textwidth]{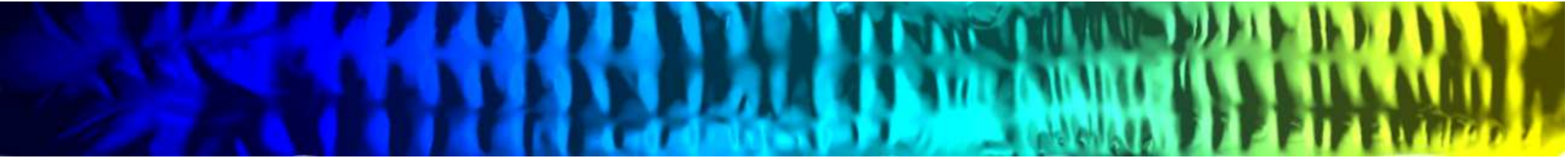}\vspace{-0.5mm}\\
(b) Registered supine and prone section 1\vspace{0.5mm}\\
\vspace{-0.5mm}\includegraphics[width=0.945\textwidth]{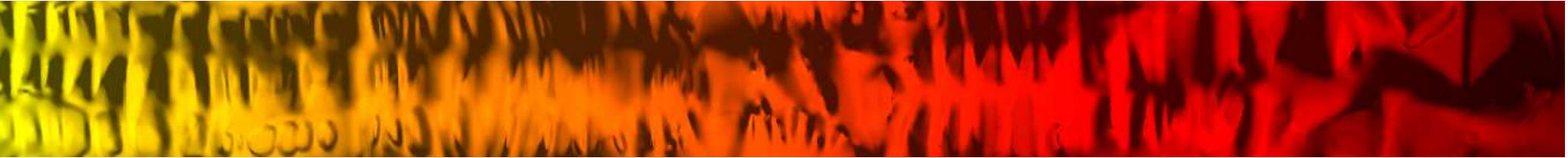}\vspace{0mm}\\
\vspace{0mm}\includegraphics[width=0.945\textwidth]{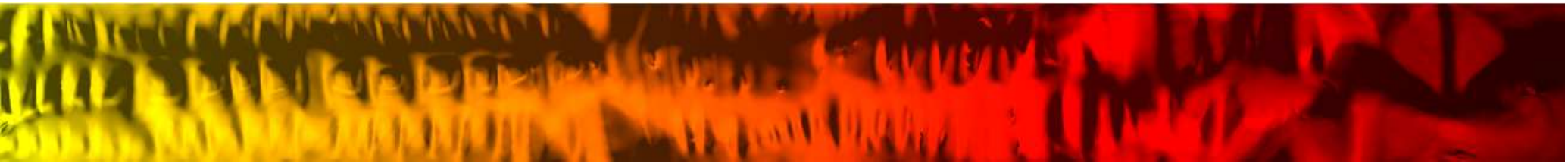}\vspace{-0.5mm}\\
(c) Unregistered supine and prone section 2\vspace{0.5mm}\\
\vspace{-0.5mm}\includegraphics[width=0.945\textwidth]{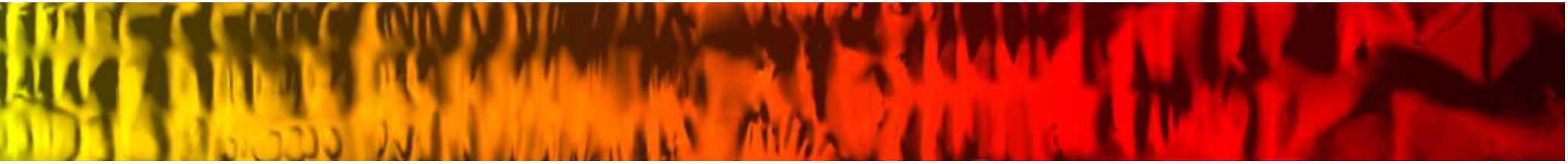}\vspace{-0.5mm}\\
\vspace{-0.5mm}\includegraphics[width=0.945\textwidth]{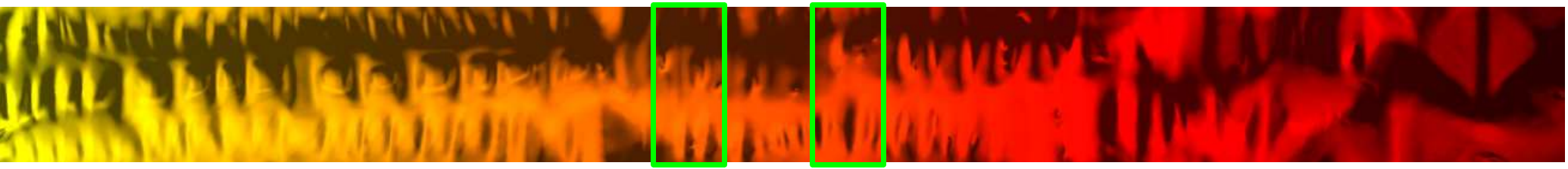}\vspace{-0.5mm}\\
(d) Registered supine and prone section 2\\
\end{tabular}
\end{center}
\vspace{-6.1mm}
%\caption{Flattened visual registration verification of the colon in Figure \ref{fig:colon_102}. (a) and (c) The unregistered supine and prone colon segments whereas (b) and (d) show the registered supine and prone sections, using the haustral folds as anatomical references, shown in Figure \ref{fig:colon_102}(c)-(d). The green boxes in (d) show the collapsed regions corresponding to the collapsed regions highlighted in Figures \ref{fig:colon_102}(b) and \ref{fig:colon_102}(d).
\caption{Flattened visual registration verification of the supine and prone colons in Figure \ref{fig:colon_102}. (a) and (c) The unregistered supine and prone sections. (b) and (d) The registered supine and prone sections. The green boxes in (d) show the collapsed regions corresponding to the collapsed regions highlighted in Figures \ref{fig:colon_102}(b) and \ref{fig:colon_102}(d).
\label{fig:registration_compare}}
\vspace{-2.8mm}
\end{figure*}

\subsection{Analytic Registration Evaluation}
\label{subsec:analytic}

To evaluate our registration results, we use the distance measurement between corresponding features located on registered colon surfaces. We use the haustral fold segmentation algorithm, detailed in Section \ref{sec:algorithm}, to find corresponding anatomical references on a given supine-prone colon pair. Based on these computed references and our manual labels, we use a subset of the segmented features for our registration and measure the distance errors on the remaining features. We generally extract more than 100 features from both pairs. For this registration error evaluation, we used 75 feature points for our registration and measure the distance errors on the remaining features.

A comparison between our method and other methods is performed using our analytic evaluation results in $\mathbb{R}^3$.  For those papers that present their distance error, we compare our results with their results in Table \ref{tab:compare}. Our method produces a registration with significantly smaller distance error between corresponding points than other methods. When comparing to the published algorithm of Lai et al. \cite{lai:2010:intra}, 3 out of our 10 colon pair datasets had flips and needed to be made consistent manually in order for their algorithm to work. Moreover, in 6 out of our 10 colon pair datasets, the Lai et al. algorithm performed worse than our global registration baseline ($\varepsilon=0$ in their case) due to the incorrect detection of the landmarks.  Further details about issues we encountered with their algorithm can be found in Appendix A.1.

\begin{table}[t]
\setlength{\tabcolsep}{2.6pt}
\centering
%\vspace{-4.5mm}
%\captionsetup{labelformat=empty}
\caption{Comparison of average distance error between our Fiedler vector approach (in bold) and other registration methods.}
\label{tab:compare}
%\begin{center}
\vspace{-3mm}
%\begin{tabular*}{0.468\textwidth}{|l|r|}
\begin{tabular*}{0.455\textwidth}{lr}
\hline
\it{Method}&\it{Dist. Error}\rule{0pt}{1.05em}\\
\hline
\hline
\bf{Fiedler Vector Representation $+$ Fold matching} & \bf{5.24 mm} \rule{0pt}{1.05em}\\
\hline
Quasi-conformal mapping \cite{zeng2010supine} & 7.85 mm \rule{0pt}{1.05em}\\
%(Zeng et al. \cite{zeng2010supine}) & 7.85 mm \\
\hline
\bf{Fiedler Vector Representation} & \bf{11.98 mm} \rule{0pt}{1.05em}\\
\hline
Centerline registration + statistical analysis \cite{li:2004:medphys} & 12.66 mm \rule{0pt}{1.05em}\\
%(Li et al., 2004 \cite{li:2004:medphys}) & 12.66 mm \\
\hline
Linear stretching/shrinking of centerline \cite{acar:2001:embs} & 13.20 mm \rule{0pt}{1.05em}\\
%(Acar et al., 2001 \cite{acar:2001:embs}) & 13.20 mm \\
\hline
Centerline feature match + lumen deformation \cite{suh:2009:jcat} & 13.77 mm \rule{0pt}{1.05em}\\
%(Suh \& Wyatt., 2009 \cite{suh:2009:jcat}) & 13.77 mm \\
\hline
Fiedler vector + piecewise registration \cite{lai:2010:intra} & 14.19 mm \rule{0pt}{1.05em}\\
\hline
Centerline point correlation \cite{devries:2006:bjr} & 20.00 mm \rule{0pt}{1.05em}\\
%(de Vries et al., 2006 \cite{devries:2006:bjr}) & 20.00 mm \\
\hline
Taeniae coli correlation \cite{huang:2005:vis} & 23.33 mm \rule{0pt}{1.05em}\\
%(Huang et al., 2005 \cite{huang:2005:vis}) & 23.33 mm \\
\hline
\end{tabular*}
%\end{center}
%\vspace{-3.5mm}
\end{table} 

\subsection{Visual Registration Verification}
\label{sec:visual_verification}

We provide a visual registration verification of our algorithm by flattening the entire supine and prone colon models and then using the segmented folds (see Figure \ref{fig:colon_102}) as anatomical references to align the supine and prone datasets. In Figure \ref{fig:registration_compare}, we show the entirety of both the supine and prone colon models, mapped to the plane, both unregistered and registered. The images of the registered segments clearly show very good alignment of the supine and prone colon structures, whereas the unregistered segments show poor alignment. Compared to previous methods, our algorithm provides a visually superior result for viewing corresponding locations on the two colon models.

We have also shown our results to a VC-trained radiologist. Due to the use of the Fiedler vector color mapping on the colon, he was able to easily find the corresponding regions across supine and prone colons along with a strong correlation between the 3D endoluminal views and the flattened views. 
\section{Conclusions and Future Work}
\label{sec:conc}

Shape registration is fundamental for shape analysis problems, especially for abnormality detection in medical applications. We have introduced an efficient framework for the registration of supine and prone colons, through the use of Fiedler vector representation, to improve the accuracy of polyp detection. Specifically, we use the Fiedler vector representation to globally register the supine and prone colon models. We then use level sets computed based on the Fiedler vector representation to detect and segment folds which are in turn used as anatomical references to locally refine the global registration results. The use of the Fiedler vector representation can help in easily visualizing the corresponding regions on 3D and 2D mappings as well as across supine and prone models. We have also proven the hot spots conjecture for modeling cylindrical topology using Fiedler vector representation, which allows our approach to be used for general cylindrical modeling and feature extraction. Furthermore, the fold segmentation allows for more consistent cuts along the colon surface which in turn provides more accurate flattened visualizations.

We have provided a thorough evaluation of our fold segmentation results by using the $SAR$ metric on 20 manually labeled datasets (10 multi-orientation colon pairs). We have compared our method with other registration algorithms based on the computed registration error metric and have found our method to provide a significant improvement. Finally, we have also provided visual verification of our results on complete supine and prone colon pairs.
%Unlike previous works, we do not make any assumptions about the quality of the data for our fold segmentation algorithms to work.

In the future, we will leverage the Fiedler vector level set approach for polyp detection in a given dataset and create a more integrated framework for computer-aided detection based on our registration results. It has been shown in previous works that the performance and accuracy of the computer-aided detection can be increased dramatically if the detection of polyps is done separately on haustral folds and endoluminal walls. We will use our fold segmentation framework to build this dichotomy and to help improve the accuracy of the CAD algorithms. We will also deploy our method in longitudinal analysis to register colons for the same patient across multiple visits, rather than just for single patient visits. Moreover, in the future, we will also cater to the severe under-distended colon cases where the segmentation can result in more than one colon segment per patient dataset. These cases can vary in complexity from a simpler (single supine/prone colon segment being registered to multiple supine/prone colon segments) to a more cumbersome (multiple supine/prone colon segments being registered to multiple supine/prone colon segments) scenario.

\acknowledgments{The datasets are courtesy of Stony Brook University Hospital (SBUH) and Dr. Richard Choi, Walter Reed Army Medical Center. We would like to thank Dr. Kevin Baker of SBUH for his help in this project. This work has been partially supported by NSF grants DMS-1418255, CNS-0959979, IIP-1069147,  CNS-1302246, and the Marcus Foundation.
}

\begin{figure*}[t!]
\begin{center}
\begin{tabular}{c}
\includegraphics[width=0.98\textwidth]{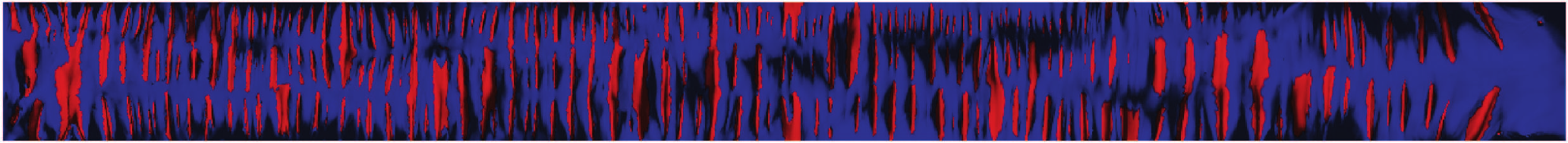}\\
(a) Fold detection algorithm using heat diffusion and fuzzy C-means clustering \cite{chowdhury:2010:colonic}\\
\includegraphics[width=0.98\textwidth]{figures/folds_seg/colon_160_S_simp_flatten_folds.pdf}\\
(b) Our fold detection and segmentation algorithm
\end{tabular}
\end{center}
\vspace{-6mm}
\caption{Fold detection evaluation on a flattened colon segment. (a) Fold detection algorithm using heat diffusion and fuzzy C-means clustering \cite{chowdhury:2010:colonic} and (b) our fold detection algorithm.
\label{fig:fold_detection_eval}}
\end{figure*}

\begin{figure}[t!]
\begin{center}
\begin{tabular}{cc}
\includegraphics[width=0.15\textwidth]{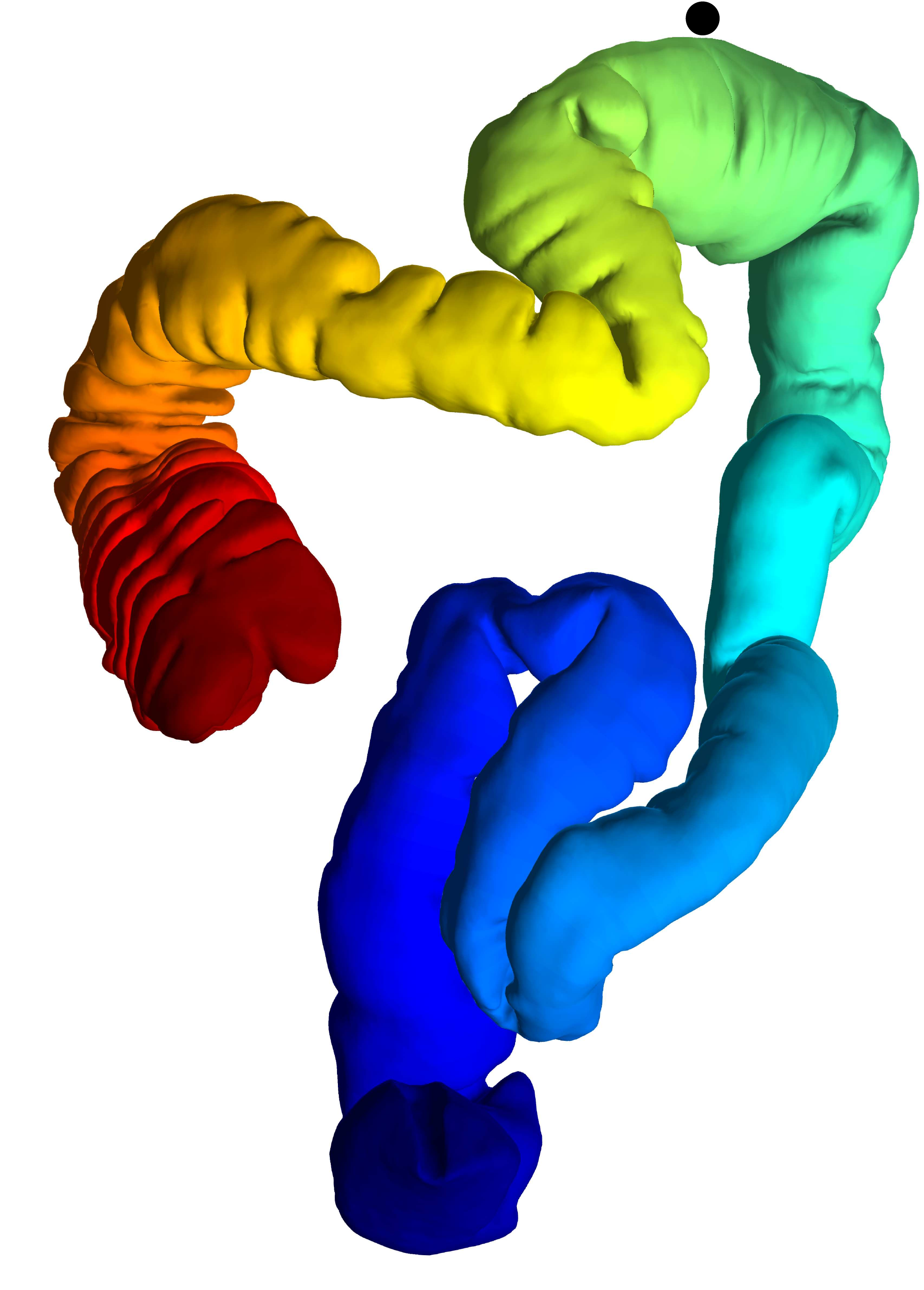}&%.24
\includegraphics[width=0.17\textwidth]{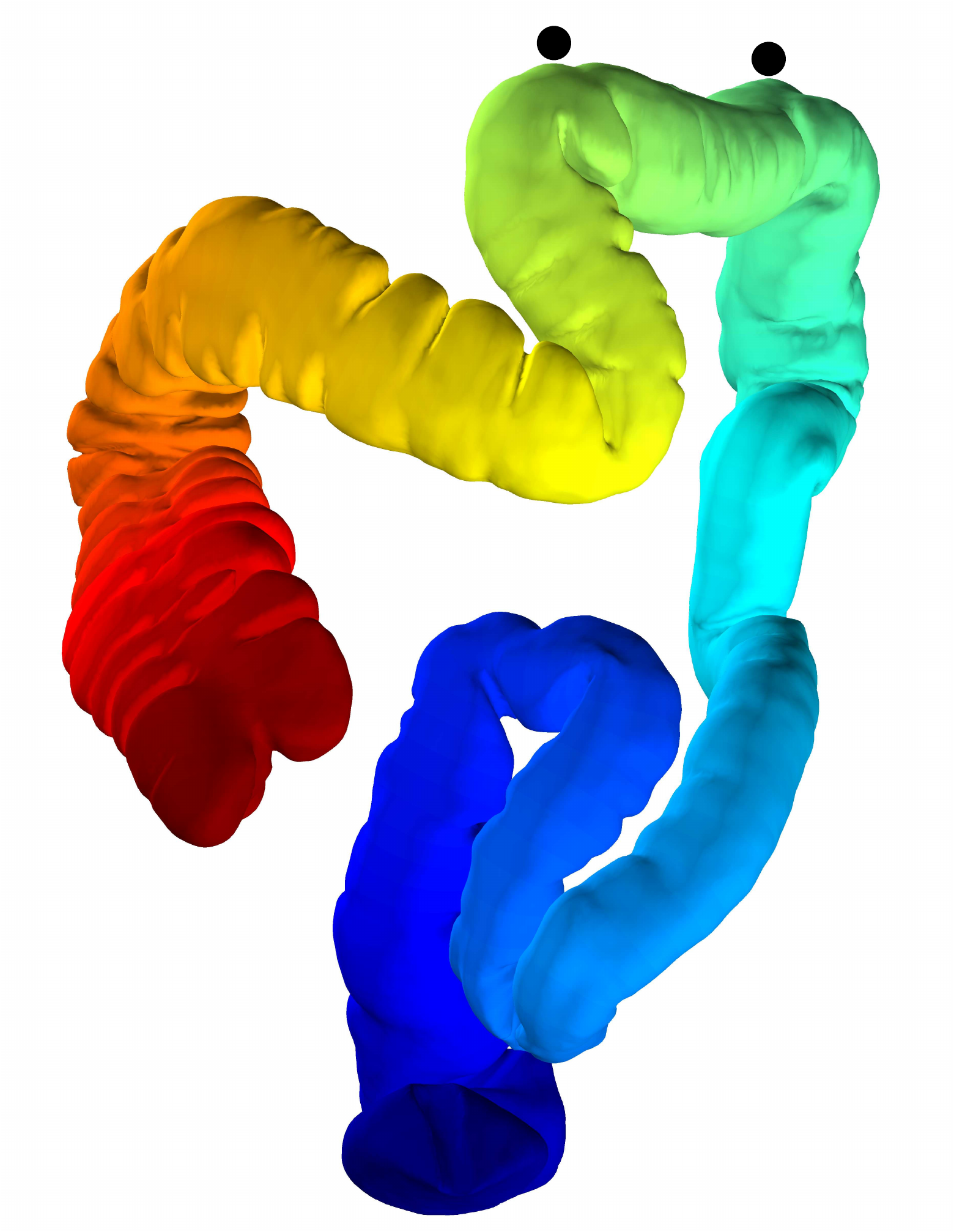}\\%.25
(a) Supine & (b) Prone\\
\includegraphics[width=0.15\textwidth]{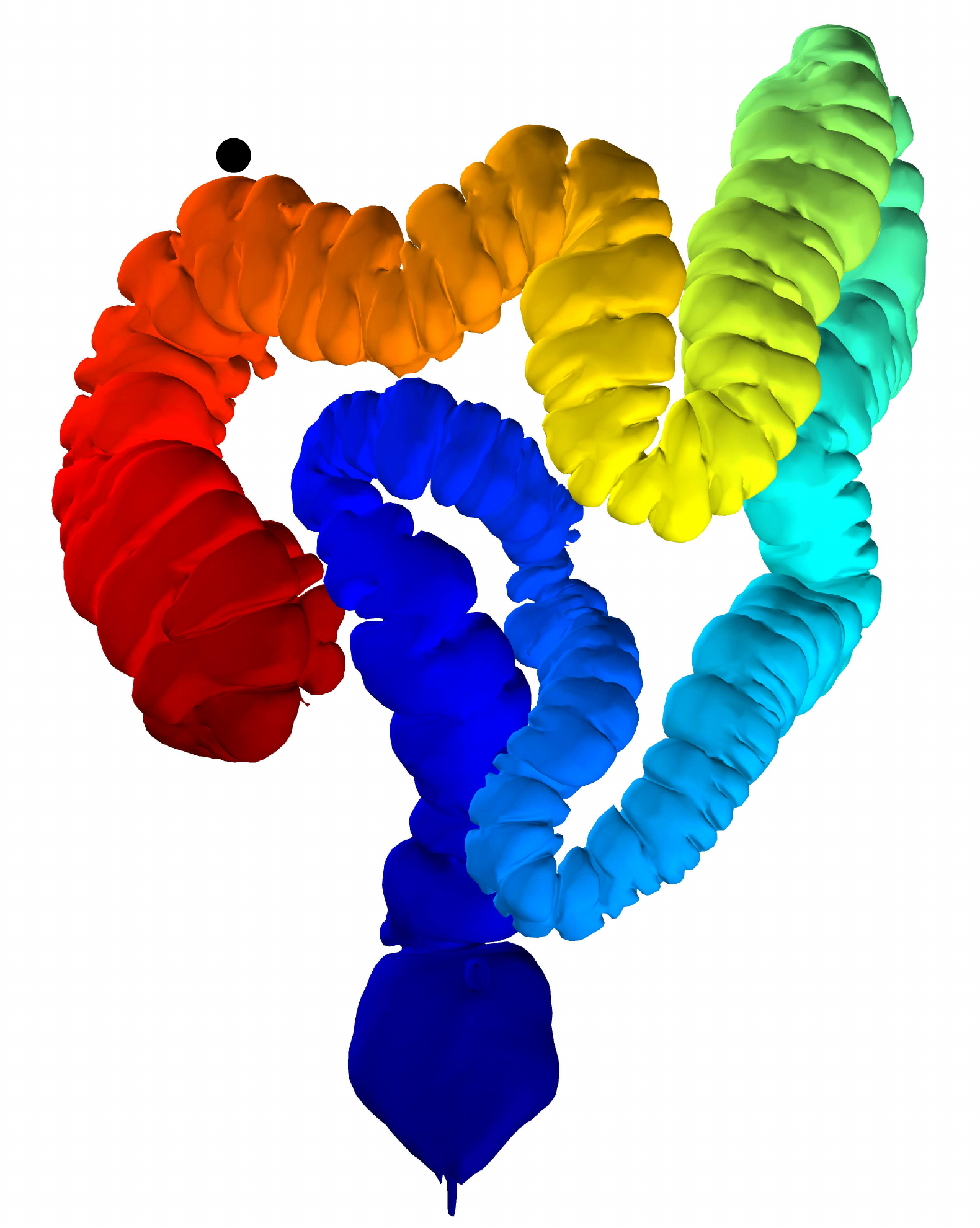}&%.24
\includegraphics[width=0.17\textwidth]{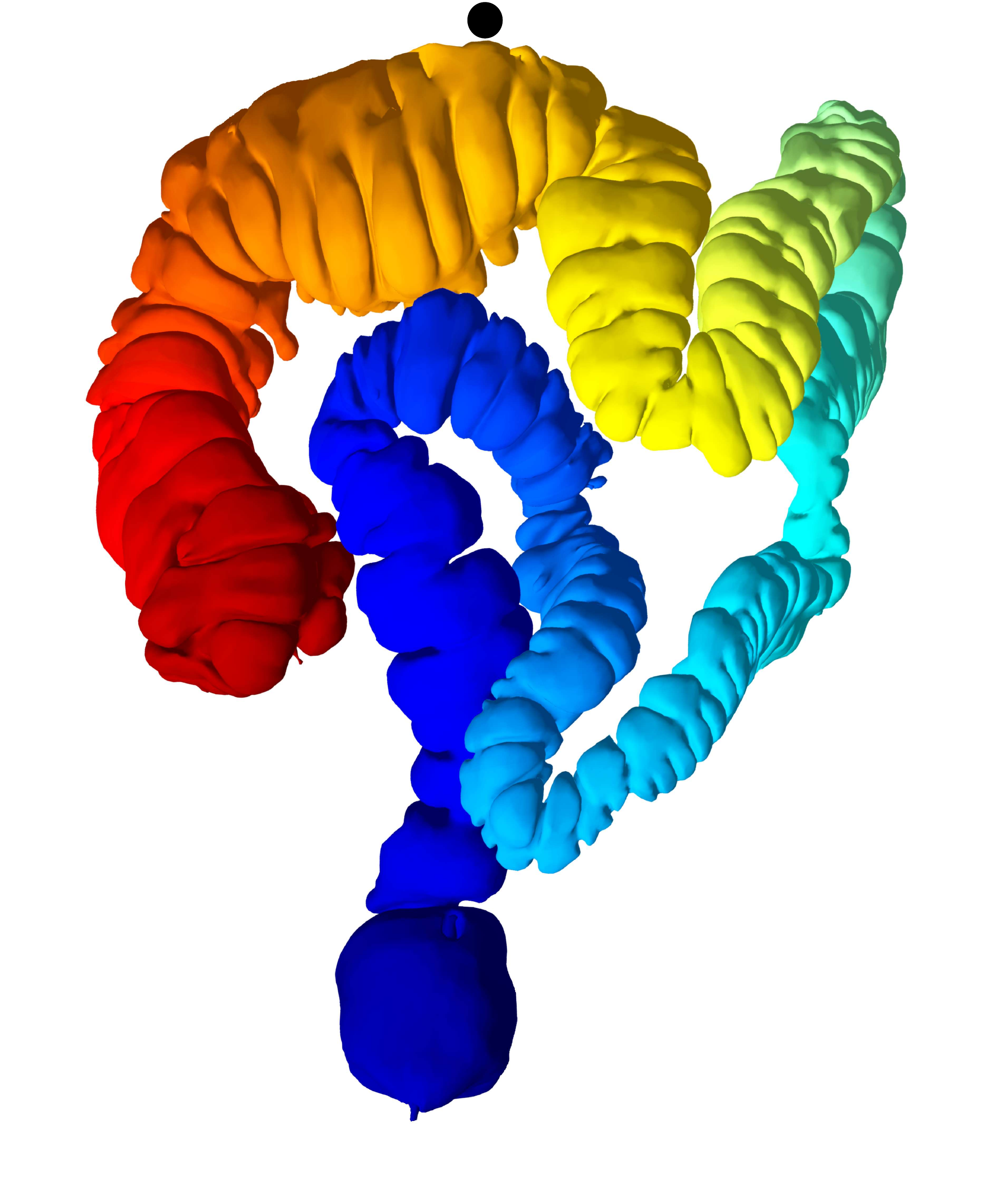}\\%.26
(c) Supine & (d) Prone\\
\includegraphics[width=0.18\textwidth]{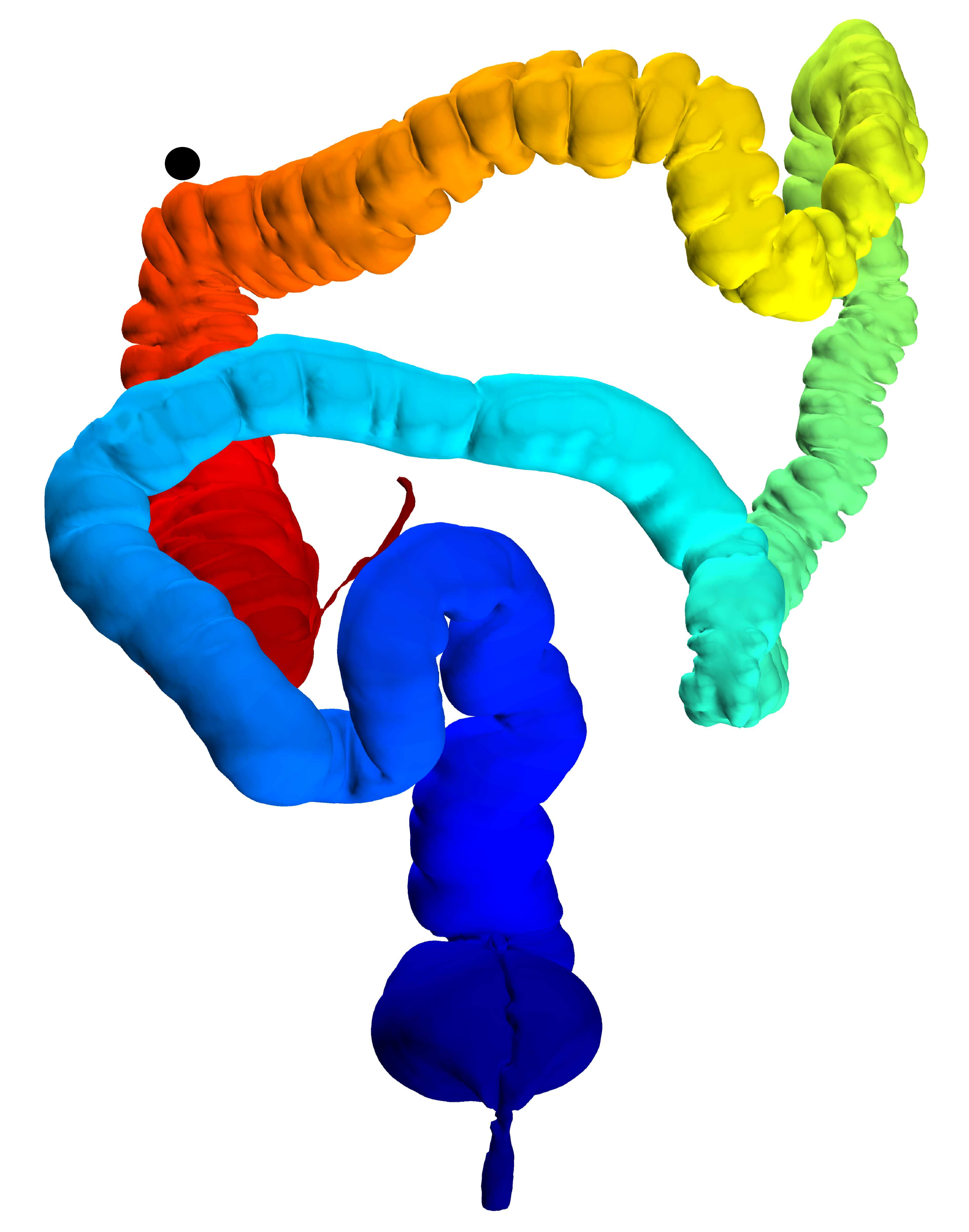}&%.28
\includegraphics[width=0.15\textwidth]{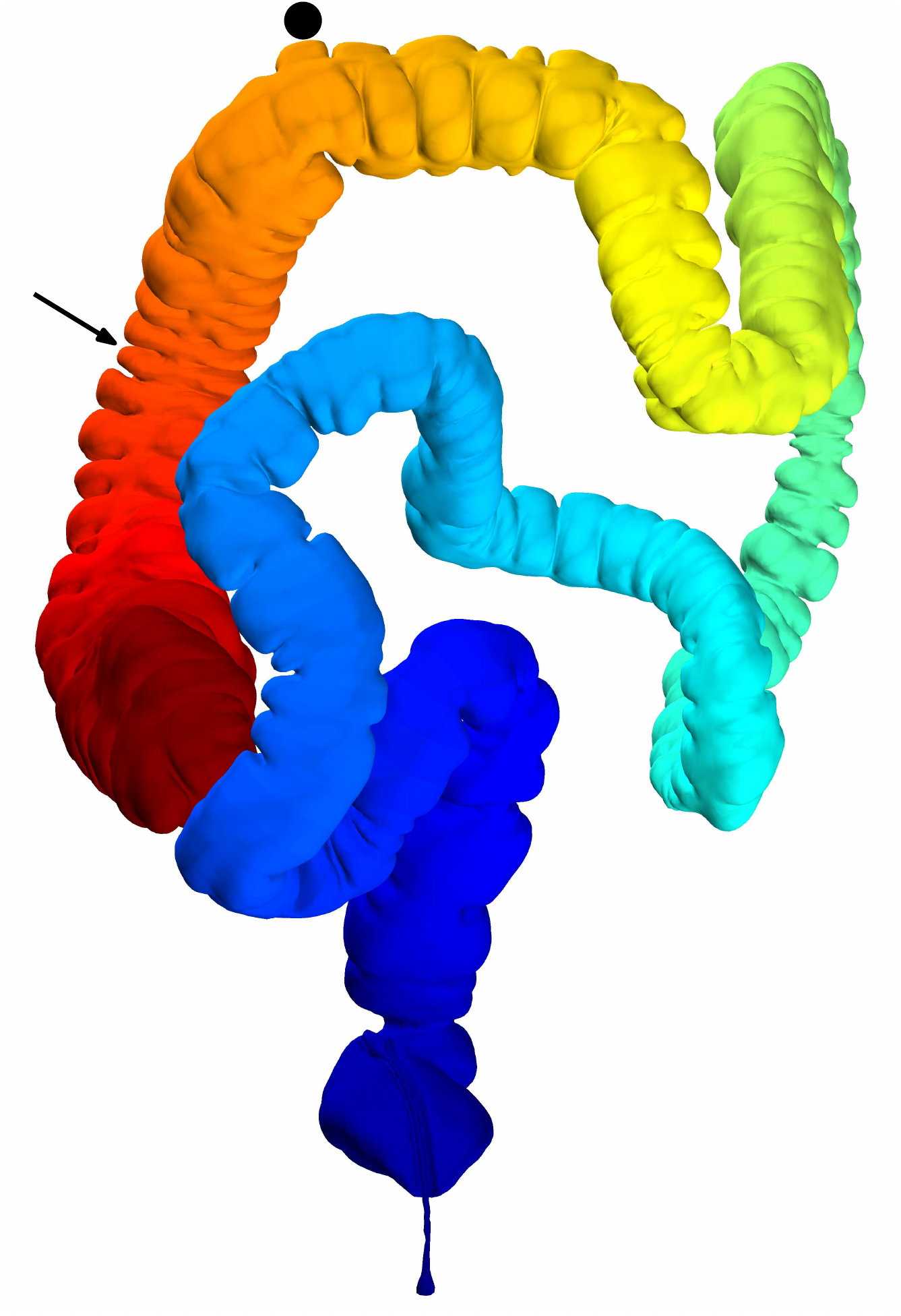}\\%.24
(e) Supine & (f) Side\\
\end{tabular}
\end{center}
\vspace{-5mm}
\caption{\emph{Registration problems with the Lai et al. \cite{lai:2010:intra} algorithm}. Due to the dependence of the Lai et al. registration on correctly detecting flexures, the average distance registration error increases with any $\varepsilon > 0$ if the flexures are not detected correctly on the corresponding colon pairs. The black circles indicate the local maximum z-coordinates as per their approach, and we use $\varepsilon=0.05$ to compute the corresponding points on the prone/side model. In (b), two local maximum z-coordinates are detected, instead of one (which is the Lai et al. assumption). In (f) the arrow shows the corresponding actual location on the side colon for the point detected in the supine model (e).
\label{fig:datasets}}
\end{figure}

% \bibliographystyle{abbrv-doi}
% \bibliography{vc_reg}

\newpage
\section*{Appendix}
\subsection*{A.1. Previous Fiedler vector colon registration}
We have compared our work against a similar method using Fiedler vector introduced by Lai et al. \cite{lai:2010:intra}. Lai et al.'s automatic registration algorithm is dependent on the detection of two landmarks, namely the splenic and hepatic flexures, once the Fiedler vector has been computed on both the supine and prone models. Theoretically, the hepatic flexure forms the topmost point of the ascending colon and the splenic flexure forms the topmost point of the descending colon; this is the assumption of the Lai et al. algorithm as well. However, on real patient colon datasets that we tested, splenic and hepatic flexure detection is a much more complex problem and cannot simply be characterized by the local maximum z-coordinate near one Fiedler vector extrema or the other (as described in Lai et al.'s paper) due to the inflection points and the large variation in the troughs and ridges of the folds at the top. In general, local maximum z-coordinates can span several folds at the top and hence, the flexures detected within an epsilon value (using Lai et al.'s published approach) on the supine and prone colon models can vary considerably.

Due to this issue on some datasets when using Lai et al.'s registration algorithm, rather than improving the results (over our global registration baseline, which is essentially the registration with epsilon value of 0 in their case) once the landmarks (splenic and hepatic flexures) are detected, the results deteriorated in six out of our ten colon pair datasets (with an average registration distance error of 14.19 mm) due to the incorrect detection of the flexures on the corresponding colon pair (supine/prone/side) datasets. The average registration distance error for our complete algorithm, including the fold matching, was 5.24 mm for the same datasets.

Some of the registration problems that we encountered with the Lai et al. algorithm are highlighted in Figure~\ref{fig:datasets} with three of our datasets. In Figure~\ref{fig:datasets}(b) for example, two local maxima are detected on the prone model within the $\varepsilon=0.05$ value. This case is not handled based on the published Lai et al. approach. To compute the results, we picked the one closer to the maximum Fiedler vector value (as done by the Lai et al. approach for the supine but not for the prone model).

Moreover, the details of Lai et al.'s published algorithm do not deal with the flipping of the Fiedler vector between the supine and prone colon pairs. If the Fiedler vector values are flipped, this can lead to the cecum-hepatic segment being registered with the rectum-splenic segment on the corresponding model, which gives an erroneous result. To avoid the flipping in our case, we manually make the Fiedler vector minimum and maximum values consistent between the colon pairs.

\subsection*{A.2. Previous fold segmentation using fuzzy C-means}
We provide a visual comparison of the detection accuracy of our algorithm with a previous fold detection algorithm using heat diffusion and fuzzy C-means clustering \cite{chowdhury:2010:colonic} in Figure  \ref{fig:fold_detection_eval}. As can be observed from a visual inspection of the results, our method succeeds in segmenting the individual folds, whereas the fuzzy C-means clustering often results in multiple folds being segmented together as a single fold.

\end{document}